\documentclass{article}

 \usepackage[preprint]{neurips_2026}
\usepackage[utf8]{inputenc} % allow utf-8 input
\usepackage[T1]{fontenc}    % use 8-bit T1 fonts
\usepackage{hyperref}       % hyperlinks
\usepackage{url}            % simple URL typesetting
\usepackage{booktabs}       % professional-quality tables
\usepackage{amsfonts}       % blackboard math symbols
\usepackage{nicefrac}       % compact symbols for 1/2, etc.
\usepackage{microtype}      % microtypography
\usepackage{xcolor}         % colors
\usepackage{amsmath}
\usepackage{amsthm}
\usepackage{algorithm}
\usepackage{algorithmicx}
\usepackage{algpseudocode}
\usepackage{graphicx}
\usepackage{wrapfig}
\usepackage{multirow}
\usepackage{siunitx}
\usepackage{graphicx}
\usepackage{booktabs}
\usepackage{caption}
\usepackage[table]{xcolor}
\usepackage{enumitem}
\usepackage{tikz}
\usetikzlibrary{arrows.meta,positioning,calc}

\definecolor{nipsgreen}{RGB}{34,139,34}
\hypersetup{
  colorlinks=true,
  citecolor=nipsgreen,
  linkcolor=nipsgreen,
  urlcolor=nipsgreen
}

\newtheorem{theorem}{Theorem}[section]
\newtheorem{proposition}[theorem]{Proposition}
\newtheorem{lemma}[theorem]{Lemma}

\title{GenPO++: Generative Policy Optimization with Jacobian-free Likelihood Ratios}

\author{%
  Ke Hu$^{1 *}$ \hspace{1em} Shutong Ding$^{1}$\thanks{Equal contribution. $\dag$Corresponding author.} \hspace{1em}  Panxin Tao$^{1}$\hspace{1em} \textbf{Jingya Wang$^{1}$ \hspace{1em}  Ye Shi$^{1\dag}$ } \\ 
  \vspace{1pt}\\
  $^1$ShanghaiTech University \\
  \vspace{1pt}\\
  \texttt{ \{huke2024, dingsht, taopx2022\}@shanghaitech.edu.cn} \\
  \texttt{\{wangjingya, shiye\}@shanghaitech.edu.cn} \\
}

\begin{document}

\maketitle

\vspace{-20pt}
\begin{abstract}

Generative policies provide expressive and multimodal action distributions, making them attractive for reinforcement learning (RL) in complex continuous-control tasks. Among them, flow-based policies are especially appealing because they generate actions through deterministic transport maps. However, applying such generative policies to likelihood-based on-policy learning remains limited by the difficulty of evaluating the probability of executed actions. Existing flow RL methods either replace the true action-density ratio with approximate surrogates, which can introduce biased updates, or recover exact likelihoods through dummy-action augmentation, which enlarges the policy space and increases computation. In this work, we propose GenPO++, a reversible generative policy optimization framework that uses history states as auxiliary memory in a high-order reversible ODE solver, yielding exact inversion without changing the original action dimension. The resulting generative policy map has a log-determinant determined only by fixed solver coefficients, enabling exact and Jacobian-free likelihood-ratio computation. This design preserves the expressiveness of generative flow policies while avoiding both action ratio bias and dummy-action overhead. We evaluate GenPO++ on large-scale simulated control, fine-tuning, and real-world robotic manipulation tasks, where it achieves competitive or superior performance over state-of-the-art on-policy RL methods, while improving training stability and computational efficiency. 

\end{abstract}

\section{Introduction}

\begin{figure}
    \centering
    \includegraphics[width=1.0\linewidth]{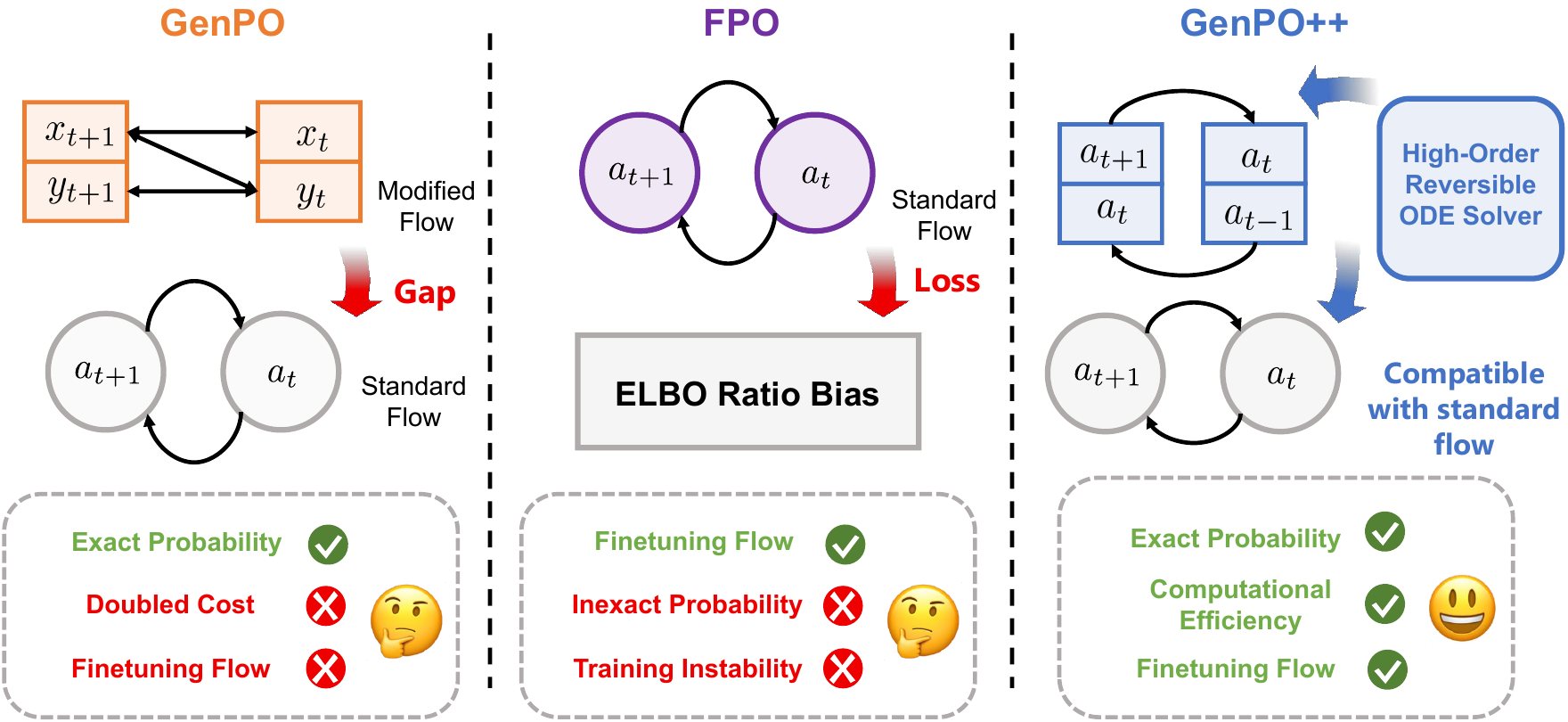}
    \caption{Comparison of FPO, GenPO, and GenPO++. FPO relies on an ELBO surrogate ratio that can be biased, while GenPO obtains exact inversion through dummy-action augmentation at the cost of an enlarged action space. GenPO++ replaces dummy actions with solver-history states, achieving exact inversion and Jacobian-free likelihood-ratio computation while preserving the original action dimension.}
    \vspace{-10pt}
    \label{fig:genpo_comparison}
\end{figure}

Reinforcement learning (RL) has achieved strong performance in continuous control, often with simple Gaussian policies and stable likelihood-based updates such as PPO~\cite{schulman2017proximal}. However, Gaussian policies are limited in expressiveness: they typically represent unimodal action distributions and can struggle in tasks with multimodal action choices, contact-rich dynamics, or discontinuous action manifolds. Generative policies based on diffusion and flow models offer a promising alternative. Diffusion Policy~\cite{chi2023diffusion}, Diffuser~\cite{janner2022planning}, diffusion policies for offline RL~\cite{wang2022diffusion,ajay2022conditional}, and flow-matching imitation policies~\cite{rouxel2024flow} have shown that generative models can represent rich conditional action distributions and capture diverse behavior modes. Flow matching and rectified flow further provide efficient deterministic samplers, making them attractive for scalable control and policy fine-tuning~\cite{lipmanflow,lipman2022flow,liu2022flow}.

Despite this progress, turning a good generative sampler into an effective on-policy policy remains difficult. The key obstacle is not action generation itself, but evaluating the probability of an already executed action under a new policy. This state-action likelihood ratio is required by clipped policy updates, KL control, and entropy regularization. For Gaussian policies, the ratio is closed form. For diffusion and flow policies, the sampler is usually a discretized neural ODE or denoising process: exact inversion can be nontrivial, and exact density evaluation may require neural-network Jacobian determinants or costly trace estimators~\cite{song2020score}. This creates a gap between the expressive power of generative policies and the likelihood-based machinery that makes on-policy learning stable.

Recent methods have made important progress but still leave critical limitations. Flow Policy Optimization (FPO)~\cite{mcallister2025flow} avoids explicit density computation by replacing the true action-density ratio with a tractable ELBO-based surrogate. This makes flow policies practical to optimize, but the surrogate ratio can differ from the true probability ratio when the variational gaps of the current and old policies differ. As a result, clipping and regularization are applied to an approximate objective. GenPO~\cite{ding2025genpo} takes a complementary route by using exact diffusion inversion and dummy-action augmentation to construct an invertible generative policy. This enables exact likelihood-ratio computation, but it doubles the effective action dimension, introduces redundant exploration variables, changes the standard flow/diffusion sampling procedure, and adds computational overhead. Meanwhile, exact inversion methods such as EDICT~\cite{wallace2023edict}, BDIA~\cite{zhang2024exact}, and BELM~\cite{wang2024belm} show how to make diffusion trajectories reversible for reconstruction and editing, but they do not directly provide an efficient policy update mechanism for continuous control.

In this work, we propose \textbf{GenPO++}, a reversible flow policy optimization framework for likelihood-based on-policy learning with expressive generative policies. Our key observation is that the auxiliary memory needed for exact inversion does not have to be an independent dummy action. Instead, it can be obtained from solver-history states that naturally arise in high-order ODE integration. Based on this observation, GenPO++ constructs a reversible high-order flow-policy solver whose inverse is available in closed form. More importantly, the log-determinant of the policy map is independent of the neural velocity network and depends only on fixed solver coefficients. This yields an exact, Jacobian-free likelihood ratio without changing the original action dimension.

GenPO++ has several practical advantages. First, it avoids dummy-action augmentation and therefore preserves the original action representation, which is important for fine-tuning pretrained supervised flow or diffusion policies. Second, it avoids repeated neural-network Jacobian determinant computation, reducing the learning-time overhead that appears in exact-likelihood generative policy methods. Third, by using an exact likelihood ratio rather than an ELBO surrogate, it improves training stability in settings where approximate objectives can drift from the true executed action distribution.

We evaluate GenPO++ on large-scale simulated control, policy online fine-tuning, and real-world robotic manipulation tasks. Across these settings, GenPO++ achieves competitive or superior performance compared with Gaussian PPO, diffusion policy fine-tuning, FPO, and GenPO, while improving stability and reducing inference or learning overhead. Our contributions are summarized as follows:

\begin{itemize}[leftmargin=4pt, rightmargin=4pt]
    \item \textbf{Detailed Analysis of On-Policy Generative RL.} 
    % We analyze the key limitations of existing flow-based reinforcement learning methods: approximate-ratio approaches can introduce biased updates, and exact-inversion approaches based on dummy actions enlarge the policy space and increase computational cost.
    We identify the central obstacle of applying flow-based generative policies to on-policy RL: the probability ratio of executed actions must be evaluated accurately after data collection. We formalize how approximate ELBO-ratio objectives can deviate from the true action-density ratio, and explain why dummy-action exact-inversion methods introduce redundant policy
    dimensions and additional computation.

    \item \textbf{High-Order Reversible Generative Policy.} 
    % We introduce a high-order reversible flow policy that uses past solver states as auxiliary memory, and instantiate it in on-policy RL. The new algorithm GenPO++ removes dummy-action augmentation in GenPO, which preserves the original action dimension and avoids neural-network Jacobian determinants during likelihood-ratio computation.
    We propose GenPO++, a high-order reversible generative policy that replaces independent dummy actions with history states. The resulting transition
    admits closed-form inversion, preserves the original action dimension, and has
    a fixed-coefficient log-determinant independent of the neural velocity field,
    enabling exact Jacobian-free likelihood-ratio computation. We further connect
    the update to Adams--Bashforth integration and derive its local truncation
    error.
    \item \textbf{Simulation and Real-World Evaluation.} We conduct experiments in large-scale simulated benchmarks, imitation-to-RL fine-tuning tasks, and real-world dexterous hand manipulation tasks, demonstrating improved stability, efficiency, and final performance over prior generative policy optimization methods.
\end{itemize}

\vspace{-5pt}
\vspace{-10pt}
\section{Related Works}
\vspace{-3pt}
\subsection{Diffusion and Flow-Based Generative Models}
\vspace{-3pt}
Diffusion models formulate generation as an iterative transformation from noise to data, typically by learning a reverse denoising process associated with a predefined denoising
dynamics~\cite{sohl2015deep,ho2020denoising,song2020score}. Deterministic variants and probability-flow perspectives, such as DDIM~\cite{song2020denoising} and
score-based probability-flow ODEs, further connect diffusion sampling with continuous-time
transport~\cite{song2020denoising,song2020score}. 
Flow-based generative models provide a
closely related transport view, where samples are generated by moving particles from a
simple base distribution to a target distribution through an invertible or continuous-time
mapping~\cite{rezende2015variational,dinh2016density,kingma2018glow,
chen2018neural,grathwohl2018ffjord}.
Flow matching has recently emerged as a scalable way to train continuous-time generative flows without simulation-based maximum likelihood training. Instead of optimizing exact likelihoods, it learns a velocity field by regressing conditional transport directions between source and target samples~\cite{lipman2022flow}. Conditional flow matching and optimal-transport flow matching generalize this idea to broader coupling strategies and often produce simpler transport paths~\cite{tong2023improving}. Rectified flow further emphasizes straight transport trajectories and efficient
generation~\cite{liu2022flow}, while stochastic interpolants provide a unified framework connecting diffusion, stochastic dynamics, and deterministic flows
~\cite{albergo2025stochastic}. These developments make flow models attractive for policy learning because they can represent expressive and multimodal conditional action distributions while often requiring fewer sampling steps than standard stochastic diffusion samplers.

The practical behavior of flow policies is strongly shaped by the numerical solver used to discretize the underlying continuous-time dynamics. Higher-order solvers reuse past model evaluations to improve trajectory accuracy or reduce the number of function evaluations, as shown by DPM-Solver~\cite{lu2022dpmsolver}, DEIS~\cite{zhang2022fast},
UniPC~\cite{zhao2023unipc}, and EDM-style solvers~\cite{karras2022elucidating}. Beyond sampling efficiency, the discretized solver also defines a transport map from latent noise
to generated samples, whose invertibility becomes important when likelihood evaluation, latent recovery, or policy-ratio computation is required. Classical normalizing flows, including NICE~\cite{dinh2014nice}, RealNVP~\cite{dinh2016density}, Glow~\cite{kingma2018glow}, continuous normalizing flows~\cite{chen2018neural,grathwohl2018ffjord}, and residual-flow models such as i-ResNet and Residual Flows~\cite{behrmann2019invertible, chen2019residual}, obtain tractable likelihoods by building invertibility into the model architecture. Recent diffusion inversion methods, including EDICT~\cite{wallace2023edict}, BDIA~\cite{zhang2024exact}, and BELM~\cite{wang2024belm}, instead study how to make the sampling trajectory itself reversible through coupled transformations or bidirectional solver design. These works
highlight the importance of invertible generative maps, but they are mainly designed for density modeling, sampling, or image inversion rather than on-policy likelihood-ratio optimization. In contrast, flow policies require a solver that simultaneously supports efficient action generation, stable inversion, and tractable likelihood ratios over large rollout batches.
\vspace{-5pt}
\subsection{Generative Policy for Reinforcement Learning}
\vspace{-5pt}
% Generative policies provide a more expressive alternative to Gaussian policies and have shown strong performance in imitation learning and reinforcement learning
% \cite{chi2023diffusion,reuss2023goal,rouxel2024flow}.
% In offline RL and planning, diffusion policies are used as expressive behavior models~\cite{wang2022diffusion}, trajectory generators~\cite{janner2022planning}, decision-making priors~\cite{ajay2022conditional}, high-fidelity behavior models~\cite{chen2022offline}, actor-critic policies~\cite{hansen2023idql}, and efficient offline policies~\cite{kang2024efficient}. In off-policy RL, DIPO~\cite{yang2023policy} uses Q-gradient guidance, QVPO~\cite{ding2024diffusion} uses weighted variational policy optimization, DACER~\cite{wang2024diffusion} and DIME~\cite{celik2025dime} study entropy-regularized diffusion policies, MaxEnt diffusion policies~\cite{dong2025maximum} and consistency policies~\cite{chen2024boosting} improve online exploration, and Flow Q-Learning~\cite{park2025flow} extends flow models to value-based learning. DPPO~\cite{ren2024diffusion} further studies online fine-tuning of pretrained diffusion policies. These methods demonstrate the expressiveness of generative policies, but most of them avoid exact state-action density evaluation.

Reinforcement learning aims to learn a policy through interactions with an environment
~\cite{sutton2018reinforcement}. Compared with standard Gaussian policies, generative
policies provide a more expressive class of conditional action distributions and can better
capture multimodal behaviors. They have shown strong empirical performance in imitation
learning and offline decision making, including diffusion policies for behavior cloning
~\cite{chi2023diffusion,reuss2023goal}, trajectory generation and planning
~\cite{janner2022planning,ajay2022conditional}, offline behavior modeling
~\cite{wang2022diffusion,chen2022offline}, and offline actor-critic learning
~\cite{hansen2023idql,kang2024efficient}. However, deploying generative policies in
online RL is more challenging, since the policy must be improved from reward feedback rather than supervised action labels.

Existing online generative-policy methods are predominantly developed under off-policy
RL. One line of work treats the generative policy as an expressive actor and optimizes it
with value-based or actor-critic objectives, including normalizing-flow policies combined
with SAC or TRPO~\cite{mazoure2020leveraging,tang2018boosting}, DACER
~\cite{wang2024diffusion}, consistency-policy learning~\cite{chen2024boosting}, DIME
~\cite{celik2025dime}, and Flow Q-Learning~\cite{park2025flow}. Another line leverages
the internal structure of generative models for policy improvement. DIPO
~\cite{yang2023policy} uses Q-gradient guidance, QVPO~\cite{ding2024diffusion} weights
the diffusion variational objective by value estimates, QSM~\cite{psenka2023learning}
and maximum-entropy diffusion policies~\cite{dong2025maximum} interpret the denoising
network as a score estimator of the target policy distribution, and MEow
~\cite{chao2024maximum} exploits the layer-wise structure of normalizing flows for
maximum-entropy policy learning.

Recently, several works have begun to study generative policies in on-policy RL. DPPO~\cite{ren2024diffusion} formulates the denoising process of a pretrained diffusion policy as an inner MDP and fine-tunes it with policy-gradient updates. FPO~\cite{mcallister2025flow} brings flow matching into the on-policy policy-gradient framework by constructing a surrogate ratio from the conditional
flow-matching objective, thereby avoiding exact likelihood evaluation. GenPO~\cite{ding2025genpo} introduces an invertible diffusion-policy construction to enable tractable action log-likelihoods, KL estimation, and entropy regularization for PPO-style
updates. PolicyFlow~\cite{yang2026policyflow} approximates importance ratios through velocity field variations instead of evaluating likelihoods along the full flow path. ReinFlow~\cite{zhang2025reinflow} fine-tunes flow-matching policies by injecting learnable noise into the deterministic flow trajectory, converting it into a discrete-time Markov process with tractable likelihood computation.
\vspace{-5pt}
\section{Limitation of Existing On-Policy Generative RL}
\vspace{-5pt}
A flow policy generates actions by transporting a simple base distribution to the action space through a conditional neural ODE~\cite{chen2018neural}. Given a state $s$, the continuous flow is defined as $\frac{d x_t}{dt} = v_\theta(x_t,t |s)$.
In practice, action inference is performed by a finite-step numerical solver $x_{k+1}=x_k+v_\theta(x_k,t|s)\Delta_t$, where $\Delta_t=t_{k+1}-t_k$. Therefore, the implemented policy is not only a continuous-time transport model, but also a discretized sampler that maps latent noise to executable actions.
This distinction is crucial for on-policy reinforcement learning. Methods such as PPO optimize policies through likelihood ratios. The clipped surrogate objective is
\begin{equation}
    \mathcal{L}_{\mathrm{clip}}(\theta)
    =
    \mathbb{E}_t
    \left[
        \min(
            r_t(\theta)\hat{A}_t,\,
            \mathrm{clip}\big(r_t(\theta),1-\epsilon,1+\epsilon\big)\hat{A}_t
        )
    \right],
    \label{eq:bg_ppo_clip}
\end{equation}
For Gaussian policies, this ratio is available in closed form. However, for flow policies, evaluating its exact density requires either inverting the solver trajectory or accounting for the Jacobian determinant of the transport map~\cite{song2020score}. This makes the
likelihood-ratio computation substantially more difficult than action generation itself.

Recent on-policy generative-policy methods therefore differ mainly in how they handle this
likelihood-ratio bottleneck. FPO~\cite{mcallister2025flow} avoids direct density
evaluation by replacing the true action likelihood with a variational surrogate. In its
formulation, the action likelihood admits the lower bound:
\begin{equation}
    \mathrm{ELBO}_\theta(s,a)
    =
    \log \pi_\theta(a|s)
    -
    D_{\mathrm{KL}}\!\left(
        q(z|s,a)\,\|\,p_\theta(z|s,a)
    \right),
    \label{eq:bg_fpo_elbo}
\end{equation}
where $z$ denotes the latent variable or latent path associated with the flow policy.
FPO then constructs the PPO ratio using the ELBO values,
\begin{equation}
    \rho_{\rm FPO}(s,a)
    =
    \frac{\exp(\mathrm{ELBO}_\theta(s,a))}
         {\exp(\mathrm{ELBO}_{\theta_{\mathrm{old}}}(s,a))}.
    \label{eq:bg_fpo_ratio}
\end{equation}
This makes the update tractable, but the resulting ratio is generally not the exact
action-density ratio because the variational gaps under $\theta$ and
$\theta_{\mathrm{old}}$ need not cancel. Detailed proof can be seen in the Appendix~\ref{app:elbo_ratio_bias}. When these variational gaps vary across policy updates, the ELBO ratio can deviate from the true executed-action density ratio. Consequently, the resulting clipped update may be biased away from the likelihood-ratio direction, which can contribute to instability in high-dimensional on-policy training. This phenomenon is also clearly observed in our empirical experiments.

GenPO~\cite{ding2025genpo} takes the opposite direction and constructs an exactly
invertible generative policy by augmenting the action space with dummy variables. Given
two action components $(x_t,y_t)$, one step of its alternating update can be written as
\begin{equation}
\begin{aligned}
    \tilde{x}_{t+\Delta_t}
    &=
    x_t+v_\theta(y_t,t|s)\Delta_t,
    &
    \tilde{y}_{t+\Delta_t}
    &=
    y_t+v_\theta(\tilde{x}_{t+\Delta_t},t|s)\Delta_t,
    \label{eq:bg_genpo_alternating}\\
    x_{t+\Delta_t}
    &=
    p\,\tilde{x}_{t+\Delta_t}+(1-p)\tilde{y}_{t+\Delta_t},
    &
    y_{t+\Delta_t}
    &=
    p\,\tilde{y}_{t+\Delta_t}+(1-p)x_{t+\Delta_t},
\end{aligned}
\end{equation}
while the inversion is followed as :
\begin{equation}
\begin{aligned}
\tilde{y}_{t+\Delta_t} &= (y_{t+\Delta_t}-(1-p)x_{t+\Delta_t})/p, & \tilde{x}_{t+\Delta_t} &= (x_{t+\Delta_t}-(1-p)\tilde{y}_{t+\Delta_t})/p,\\
    y_t &= \tilde{y}_{t+\Delta_t} - v_\theta(\tilde{x}_{t+\Delta_t},t)\Delta_t, & x_t &= \tilde{x}_{t+\Delta_t} - v_\theta(y_{t},t)\Delta_t, 
\end{aligned}
\end{equation}
Because both the alternating update and the mixing operation are invertible, GenPO can compute likelihoods exactly in the augmented action space with change of variable. However, GenPO is not well suited for fine-tuning tasks. Furthermore, each inference step requires two NFEs together with the computation of the Jacobian over the entire reverse process, substantially increasing computational overhead and hardware requirements. In addition, the executed action is defined as the average of two latent variables, which further reduces the interpretability of the policy.

To address these issues, we propose \textbf{GenPO++}, which builds a reversible high-order flow-policy solver using history states as auxiliary memory. The resulting map admits exact inversion and a Jacobian-free log-determinant determined by fixed solver coefficients, enabling exact and efficient likelihood-ratio optimization without changing the original action dimension. With these properties, GenPO++ achieves strong performance in both large-scale simulated benchmarks and real-world robotic tasks.
\vspace{-5pt}
\section{GenPO++}
\vspace{-5pt}
We propose GenPO++, a reversible flow policy optimization framework for likelihood-based on-policy reinforcement learning. The method proceeds as follows. First, we construct a high-order reversible flow policy by using history states as auxiliary variables, which enables closed-form inversion without augmenting the original action dimension. Second, we show that the resulting history augmented transition has a fixed-coefficient log-determinant independent of the neural velocity field, allowing exact and Jacobian-free likelihood ratio computation. Finally, we instantiate this reversible flow policy by computing the exact augmented likelihood ratio for clipped policy optimization and adaptive KL control.

\begin{figure}[ht]
    \vspace{-2pt}
    \centering
    \includegraphics[width=\linewidth]{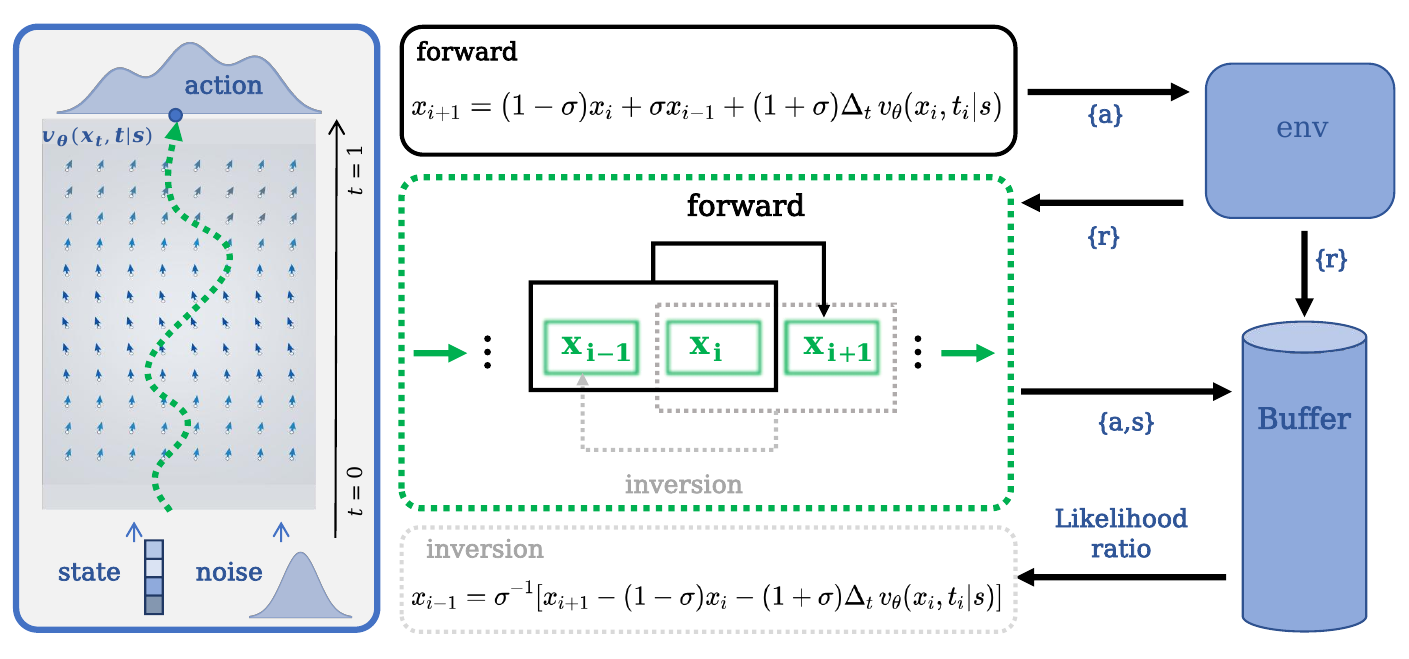}
    \caption{Pipeline of GenPO++. GenPO obtains reversibility through dummy-action augmentation, which doubles the action space and introduces redundant exploration, while GenPO++ replaces dummy actions with solver-history states in a reversible high-order flow solver, achieving exact inversion and Jacobian-free likelihood-ratio computation while preserving the original action dimension.}
    \label{fig:genpo_main}
    \vspace{-5pt}
\end{figure}

\vspace{-5pt}
\subsection{Reversible Policy via High-order Method}
\vspace{-5pt}
\label{high-order method}
We reserve \(s\) for the environment state and use \(x_i\) to denote the flow variable at solver step \(i\). 
The key difficulty in making an explicit flow solver exactly reversible is that the variable to be recovered in the inverse step is also the input used to query the velocity field. For example, reversing a standard Euler-type update requires evaluating \(v_\theta(x_i,t_i\mid s)\), while \(x_i\) itself is unknown during inversion. This coupling makes exact inversion require either a nonlinear solve or an approximation.

GenPO++ avoids this issue by decoupling the recovered variable from the velocity-query variable. Instead of evolving only a single flow variable \(x_i\), we keep a augmented pair \((x_i,x_{i-1})\). The next flow variable is constructed as a linear combination of the current variable $x_i$, the history variable $x_{i-1}$, and the velocity evaluated at the current variable $v_\theta(x_i,t_i\mid s)$. Since the output pair is \((x_{i+1},x_i)\), the inverse step already has access to \(x_i\), and can therefore evaluate \(v_\theta(x_i,t_i\mid s)\) directly and recover \(x_{i-1}\) in closed form. This gives an exactly reversible flow-policy solver without introducing independent dummy actions.
We define a transition $F_i(\cdot)$ of the companion form at step $i$ when $\sigma \neq 0$:
\begin{equation}
\begin{aligned}
\label{eq:GenPO++_update}
(x_{i+1},x_i)&=F_{i,\theta}(x_i,x_{i-1}|s)\\
    x_{i+1}
    &=
    (1-\sigma)x_i
    +
    \sigma x_{i-1}
    +
    (1+\sigma)\Delta_t\,
    v_\theta(x_i,t_i|s).
\end{aligned}
\end{equation}
This transition is nonlinear in $x_i$ through the neural velocity field, but affine in the history state $x_{i-1}$ with constant coefficient $\sigma I$. This structure makes the inverse available in closed form:
\begin{equation}
\begin{aligned}
\label{eq:GenPO++_inverse}
    (x_{i},x_{i-1})&=F^{-1}_{i,\theta}(x_{i+1},x_i|s)\\
    x_{i-1}&=\sigma^{-1}[x_{i+1}-(1-\sigma)x_i-(1+\sigma)\Delta_t\,v_\theta(x_i,t_i|s)]
\end{aligned}
\end{equation}
\begin{proposition}[High-order consistency of the reversible history update]
\label{prop:high_order_history}
The transition in Eq.~(\ref{eq:GenPO++_update}) can be interpreted as a relaxed history
approximation to a second-order Adams--Bashforth (AB2) method. Specifically, the AB2 update
can be decomposed as
\begin{equation}
    x_{i+1}^{\mathrm{AB2}}
    =
    x_i+\Delta_t\,v_i
    -
    \frac{1}{2}\Delta_t\,(v_{i-1}-v_i),
\end{equation}
Along a smooth ODE trajectory,
this correction admits the state-history approximation
\begin{equation}
    -\frac{1}{2}\Delta_t\,(v_{i-1}-v_i)
    =
    x_{i-1}-x_i+\Delta_t\,v_i
    +
    O(\Delta_t^3).
\end{equation}
Therefore, the GenPO++ update
\begin{equation}
    \hat x_{i+1}
    =
    x_i+\Delta_t\,v_i
    +
    \sigma\bigl(x_{i-1}-x_i+\Delta_t\,v_i\bigr)
\end{equation}
is a relaxed high-order history correction.
When $\sigma=1$, the update matches the AB2 correction up to an
$O(\Delta_t^3)$ local approximation error; when $0<\sigma<1$, it relaxes the
high-order correction while preserving the reversible state-history structure. Details see Appendix~\ref{AB2_GENPO++}.
\end{proposition}
Next we quantify the numerical accuracy of this approximation by characterizing its local truncation error (LTE)~\cite{wang2024belm} and making explicit how the history coefficient $\sigma$ controls the deviation from the underlying ODE trajectory.
\begin{theorem}[Local truncation error of the state-history update]
\label{thm:GenPO++_lte}
Assume that $v_\theta(x,t|s)$ is sufficiently smooth along the exact
ODE trajectory. For the GenPO++ update Eq.~(\ref{eq:GenPO++_update}), the standard local truncation error under trajectory-consistent history satisfies the following relation, see Appendix~\ref{LTE}:
\begin{equation}
    {x}_{i+1}-x^\star_{i+1}
    =
    \frac{\sigma-1}{2}\Delta_t^2
    \frac{d}{dt}v(x^\star_i, t_i|s)
    +O(\Delta_t^3).
\end{equation}
\end{theorem}
\vspace{-10pt}
\subsection{Efficient generative policy optimization}
\vspace{-3pt}
% Given the reversible map in section~\ref{high-order method}, the augmented policy density can be evaluated
% by change of variables. 
% The key property of GenPO++ is that the Jacobian determinant of each solver transition is independent of the neural velocity field.

Given the reversible transition in Section~\ref{high-order method}, likelihood evaluation reduces to a
change-of-variables computation over the augmented solver trajectory. The crucial
property of GenPO++ is that this computation does not require differentiating through
the neural velocity field. Although each transition is nonlinear in the current flow
variable through \(v_\theta\), its companion-form dependence on the history variable
makes the Jacobian determinant a fixed coefficient. 
For compactness, we write the paired solver variable as \(z_i=(x_i,x_{i-1})\).
\paragraph{Jacobian-free likelihood.}
The Jacobian of the forward map (\ref{eq:GenPO++_update}) with respect to this paired variable is given by:
\begin{equation}
    J_i
    =
    \frac{\partial F_{i,\theta}}{\partial z_i}
    =
    \begin{bmatrix}
(1-\sigma)I + (1+\sigma) \nabla_x v_\theta(x_i,t_i|s) & \sigma \\
    I & 0
    \end{bmatrix}.
    \label{eq:companion_jacobian}
\end{equation}
The determinant of this companion-form matrix is constant, $\det J_i
    =
    (-1)^d \sigma^d,$
where $d$ is the dimension of the action space. This implies that the likelihood ratio ultimately depends only on the noise variable, and not on the specific forward generation process. As long as a given action can be inverted back to the noise space, the exact ratio can be computed directly.
\begin{equation}
   \tilde r(\theta) = \frac{\tilde \pi_\theta(x_1, x_{1-\Delta_t}|s)}{\tilde \pi_{\rm old}(x_1, x_{1-\Delta_t}|s)}=  \frac{\tilde \pi_\theta(F_{0,\theta}\dots F_{1,\theta}(z_{0,\theta})|s)}{\tilde \pi_{\rm old}(F_{0,\rm old}\dots F_{1,\rm old}(z_{0,\rm old})|s)}=\frac{\tilde p(z_{0,\theta}|s)\prod \det J_i}{\tilde p(z_{0,\rm old}|s)\prod \det J_i}=\frac{\tilde p(z_{0,\theta}|s)}{\tilde p(z_{0,\rm old}|s)}
    \label{eq:latent_ratio}
\end{equation}
where $\tilde p \sim \mathcal{N}(0,I)$. Although GenPO++ computes the likelihood ratio using the joint distribution of two consecutive states, its interaction with the environment still occurs through $x_1$. We further show that this optimization objective is consistent with the standard reinforcement learning objective. Details see appendix~\ref{app:augmented_ratio_bound}. 
\begin{proposition}[Consistency of augmented-action likelihood ratios]
For any state $s$, let $\tilde{\pi}_\theta(x_1,x_{1-\Delta_t}|s)$ be the GenPO++ augmented policy
and $\pi_\theta(x_1|s)=\int \tilde{\pi}_\theta(x_1,x_{1-\Delta_t}|s)\,dx_{1-\Delta_t}$ be its executed-action
marginal.
Then, for any advantage function depending only on the executed action $x_1$,
\begin{equation}
\mathbb{E}_{(x_1,x_{1-\Delta_t})\sim \tilde{\pi}_{\theta_{\rm old}}(\cdot,\cdot|s)}
\left[
\tilde r_\theta(x_1,x_{1-\Delta_t}|s) A^{\pi_{\rm old}}(s,x_1)
\right]
=
\mathbb{E}_{x_1\sim \pi_{\theta_{\rm old}}(\cdot|s)}
\left[
r_\theta(x_1|s) A^{\pi_{\rm old}}(s,x_1)
\right].
\end{equation}
\end{proposition}
\vspace{-15pt}
\subsection{Practical Implementation}
\vspace{-5pt}
\label{sec:practical_optimization}
During rollout, the policy samples an augmented base variable \(z_0\sim \tilde p_0\), generates the
terminal augmented solver state \(z_1=(x_1,x_{1-\Delta_t})\) through the reversible
solver in Eq.~(6), and executes only the first component \(a=x_1\) in the environment. During policy updates, to enable adaptive scheduling of the learning rate, we use the augmented likelihood ratio to monitor policy deviation during PPO updates. The empirical augmented KL is estimated as
\begin{equation}
\label{eq:augmented_kl}
    \widehat{D}_{\rm KL}^{\rm aug}
    =-
    \mathbb{E}_{(x_1,x_{1-\Delta_t})\sim\tilde\pi_{old}(x_1,x_{1-\Delta_t}|s)}
    \left[
        \log \tilde r_\theta
    \right].
\end{equation}
Since marginalization contracts KL divergence, the augmented KL upper-bounds the
corresponding KL over executed actions. We use it as a conservative proxy for adaptive learning-rate control, and use a slightly larger target KL in high-dimensional tasks to avoid overly conservative
updates that may prematurely limit exploration and policy improvement. The policy is updated
with the standard clipped surrogate objective:
\begin{equation}
\mathcal{L}^{\mathrm{GenPO++}}_{\rm clip}(\theta) 
= \mathbb{E}_{(x_1, x_{1-\Delta_t}) \sim \tilde{\pi}_{\mathrm{old}}( \cdot|s)}[
\min(
\tilde{r}(\theta) \, A^{\pi_{\mathrm{old}}}(x_1),\;
\mathrm{clip}(\tilde{r}(\theta), 1-\epsilon, 1+\epsilon)\, A^{\pi_{\mathrm{old}}}(x_1)
)
].
\label{eq:GenPO++_fiber_style}
\end{equation}
Finally, since GenPO++ interacts with the environment using the terminal ODE state
\(x_1\), it avoids the dummy-action space introduced by GenPO and therefore
does not require an additional compression loss to suppress ineffective exploration.

\vspace{-5pt}
\section{Experiments}
\vspace{-5pt}
\begin{figure}
    \centering
    \includegraphics[width=0.9\linewidth]{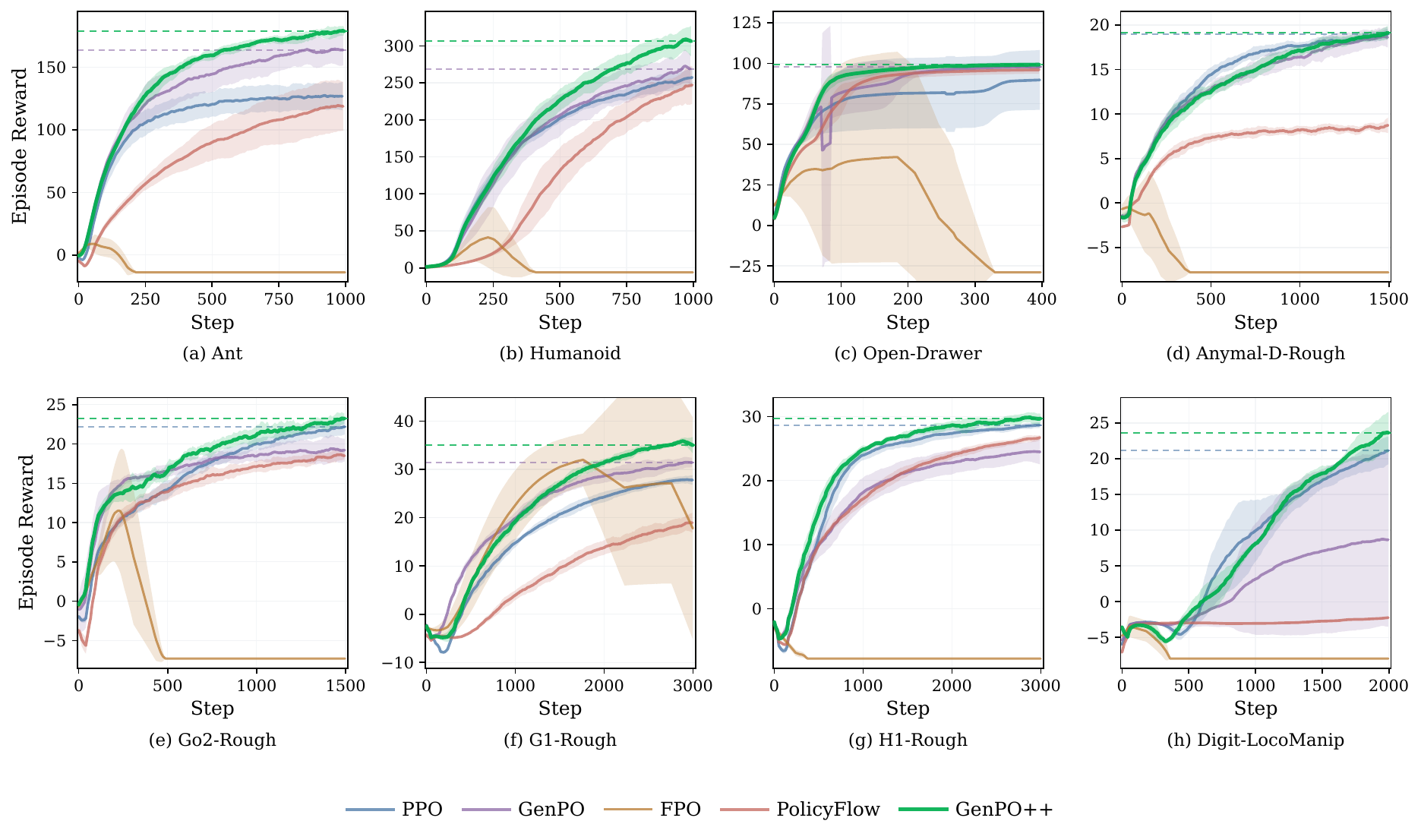}
    \caption{Learning curves across 8 IsaacLab benchmarks. Results are averaged over 5 runs. The x-axis denotes iterations, and the y-axis shows average episodic return with one standard deviation shaded.}
    \label{fig:isaaclab}
    \vspace{-10pt}
\end{figure}
In this section, we demonstrate the performance of GenPO++ in three on-policy RL scenarios: fine-tuning, learning from scratch, and real-world manipulation. These settings are chosen to test the main claims: i) whether exact reversible likelihood-ratio optimization can improve
pretrained generative policies, ii) remain stable when learning high-dimensional control policies from scratch, and iii) reduce the computational overhead of prior exact-likelihood generative policy methods.
We compare against Gaussian PPO and representative generative-policy baselines, including
DPPO~\cite{ren2024diffusion}, FPO~\cite{yi2026flow,mcallister2025flow}, GenPO~\cite{ding2025genpo}, and PolicyFlow~\cite{yang2026policyflow}. Unless otherwise stated, all methods use the same
environment configuration, rollout budget, and critic architecture. We report means over
five random seeds with standard-error shading.
\vspace{-3pt}
\subsection{Learning from Scratch on IsaacLab}
\vspace{-3pt}
\begin{wrapfigure}{r}{0.4\textwidth}
    \vspace{-10pt}
    \centering
    \includegraphics[width=0.4\textwidth]{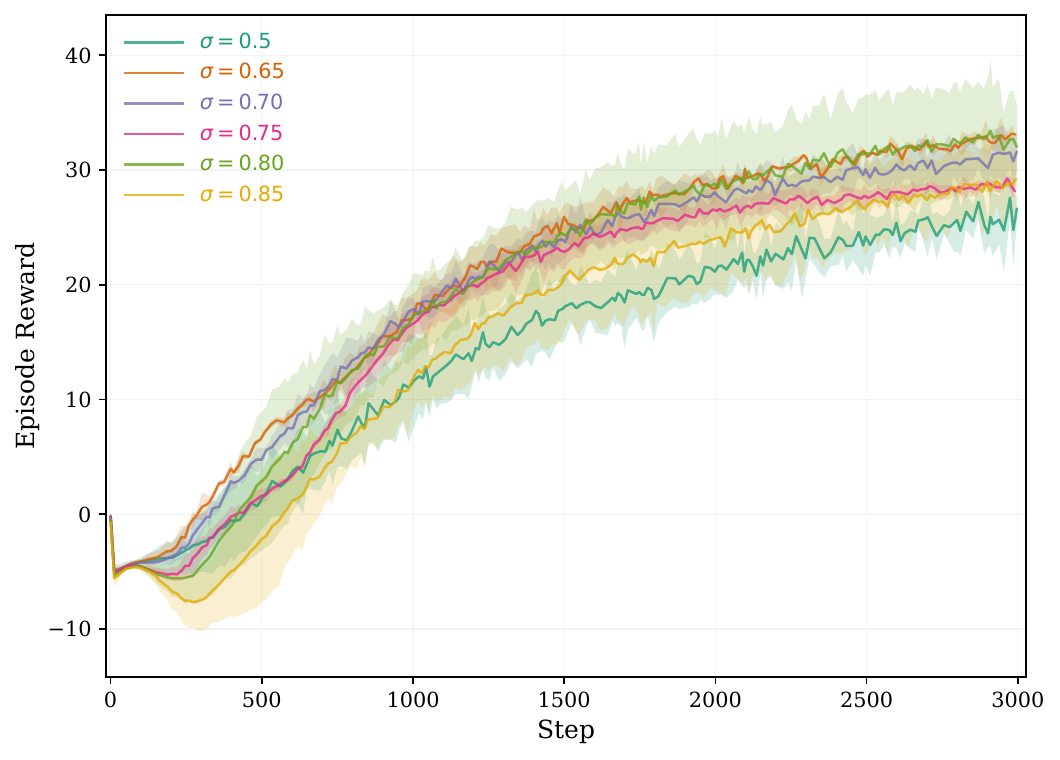}
    \caption{Ablation experiments of $\sigma$.}
    \label{fig:humanoid_single}
    \vspace{-10pt}
\end{wrapfigure}
We next evaluate methods in IsaacLab across locomotion, manipulation, and whole-body control tasks, including \textsc{Ant}, \textsc{Humanoid},
\textsc{Open-Drawer}, \textsc{Anymal-D-Rough}, \textsc{Go2-Rough}, \textsc{G1-Rough},
\textsc{H1-Rough}, and \textsc{Digit-LocoManip}.

Figure~\ref{fig:isaaclab} shows that GenPO++ achieves superior returns across the benchmark. In contrast, FPO often exhibits
large variance and collapse, suggesting that approximate likelihood-ratio
objectives can be unstable in high-dimensional control. 
Table~\ref{tab:timing_comparison} compares rollout collection and policy learning time on
IsaacLab \textsc{Humanoid}. The collection time is similar across flow-based methods, whereas update costs vary greatly. FPO is expensive due
to its heavier optimization schedule, and GenPO is slow because exact likelihood evaluation
requires Jacobian-related computation through the inverse process. GenPO++ substantially reduces total learning time compared with GenPO under the reported training configuration, while maintaining comparable rollout collection time among flow-based methods. This is more obvious in the Figure~\ref{fig:isaaclabtime}.
Figure~\ref{fig:humanoid_single} studies the effect of the coefficient
\(\sigma\). GenPO++ achieves stable performance across the tested range, suggesting
that its empirical behavior is robust to moderate variations of this parameter.
\subsection{Manipulation Fine-tuning on Robomimic}
\vspace{-3pt}
\begin{figure}
    \centering
    \includegraphics[width=0.95\linewidth]{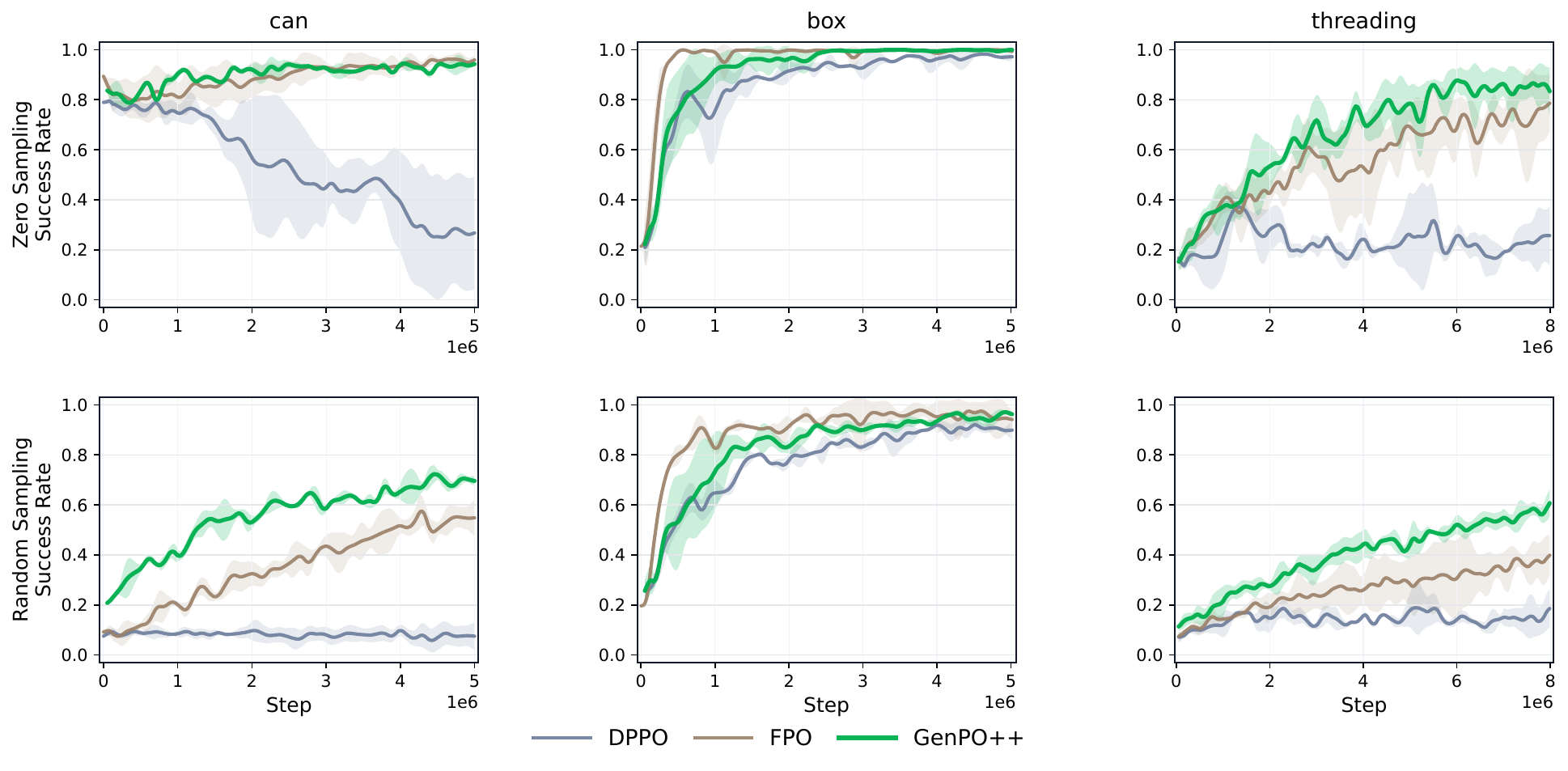}
    \caption{Online fine-tuning results on three Robomimic benchmarks. Top row reports zero-noise sampling success rate, and bottom row reports random-noise sampling success rate. Curves are averaged over 5 runs with one standard error shaded.}
    \vspace{-15pt}
    \label{fig:robomimic}
\end{figure}
We first study online fine-tuning of pretrained flow-matching policies on three Robomimic
manipulation tasks: \textsc{Can}, \textsc{Box}, and \textsc{Threading}. Each method is initialized from the same pretrained checkpoint and fine-tuned with online rewards. We evaluate both zero-noise
sampling, which measures the learned action mode, and random-noise sampling, which tests
the full stochastic policy distribution.

Figure~\ref{fig:robomimic} shows that GenPO++ consistently improves pretrained flow policies across all three Robomimic tasks. It preserves strong zero-sampling performance on \textsc{Can}, reaches high success quickly on \textsc{Box}, and achieves the best final performance on \textsc{Threading} under both zero and random sampling. This indicates that exact reversible likelihood-ratio optimization helps GenPO++
% \begin{wraptable}{r}{0.4\textwidth}
%     % \vspace{-8pt}
%     \centering
%     \caption{
%     Comparison of training efficiency. Collection and learning time are reported in humanoid env.
%     }
%     \label{tab:timing_comparison}
%     \scriptsize
%     \setlength{\tabcolsep}{3pt}
%     \renewcommand{\arraystretch}{0.9}

%     \begin{tabular}{ll S[table-format=1.3] S[table-format=1.3]}
%     \toprule
%     \textbf{Category} & \textbf{Method}
%     & {\textbf{Collection}}
%     & {\textbf{Learning}} \\
%     \textbf{} & \textbf{}
%     & {\textbf{Time (min)}}
%     & {\textbf{Time (min)}} \\
%     \midrule
%     Gaussian-based
%     & PPO & 11.63 & 1.62 \\
%     \midrule
%     \multirow{4}{*}{Flow-based}
%     & FPO & 14.06 & 58.0 \\
%     & PolicyFlow & 13.83 & 2.1 \\
%     & GenPO & 13.6 & 118.7 \\
%     & \textbf{GenPO++} & 13.68 & {\bfseries 7.1} \\
%     \bottomrule
%     \end{tabular}
%     \vspace{-10pt}
% \end{wraptable}
% 
\begin{wraptable}{r}{0.45\textwidth}
    \centering
    \caption{
    Comparison of training efficiency. Training time are reported in humanoid env.
    }
    \label{tab:timing_comparison}
    \scriptsize
    \setlength{\tabcolsep}{3pt}
    \renewcommand{\arraystretch}{0.9}

    \begin{tabular}{@{}ll S[table-format=2.4] S[table-format=3.6]@{}}
    \toprule
    \textbf{Category} & \textbf{Method}
    % & \multicolumn{1}{c}{\textbf{Collection}}
    & \multicolumn{1}{c}{\textbf{Training}} \\
    &
    % & \multicolumn{1}{c}{\textbf{Time (min)}}
    & \multicolumn{1}{c}{\textbf{Time (min)}} \\
    \midrule
    Gaussian-based
    & PPO & 13.25 \\
    \midrule
    \multirow{3}{*}{Flow-based}
    & FPO & 72.06\\
    %& PolicyFlow & 13.83 & 2.10 \\
    & GenPO & 132.30 \\
    \rowcolor{nipsgreen!12}
    & \textbf{GenPO++} & \bfseries 20.78 \\
    \bottomrule
    \end{tabular}
    \vspace{-10pt}
\end{wraptable}
improve the learned action distribution during online fine-tuning, leading to more
reliable stochastic policy improvement.
\vspace{-3pt}

\vspace{-3pt}
\subsection{Real-World Dexterous Hand Manipulation}
\vspace{-3pt}
Finally, we deploy GenPO++ on a real-world dexterous manipulation task using the RobotEra Xhand platform. The task requires the hand to rotate and loosen a nut from a bolt through repeated in-hand contact, which involves hardware actuation errors that are difficult to model accurately in simulation. We follow the simulation training framework of \cite{hsieh2025learning}, and instantiate the algorithm from \cite{qi2022hand} by replacing its PPO optimizer with GenPO++.

\begin{wrapfigure}{r}{0.4\textwidth}
    \centering
    \vspace{-10pt}
    \includegraphics[width=\linewidth]{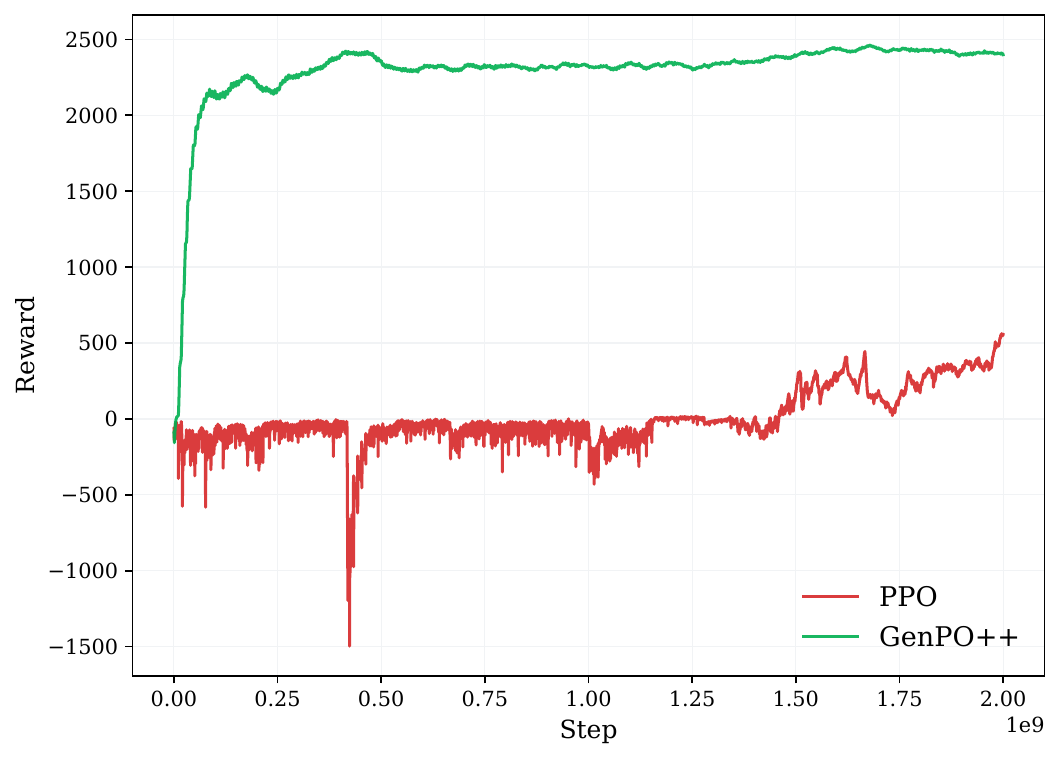}
    \caption{Episode rewards of GenPO++ and PPO during dexterous hand manipulation training.}
    \label{fig:real_world_episode_rewards}
    \vspace{-5pt}
\end{wrapfigure}

% \begin{figure}[ht]
%     \centering
%     \includegraphics[width=\linewidth]{pictures/genpo_ppo_episode_rewards.pdf}
%     \caption{Caption}
%     \label{fig:real_world_episode_rewards}
% \end{figure}

% Figure~\ref{fig:real_world_episode_rewards} compares the episode rewards during simulation training. Under the same task setup and reward design, GenPO++ achieves a much higher episode reward than PPO after fewer environment steps, showing a clear advantage in sample efficiency. PPO improves more slowly and remains at a substantially lower reward under the same training budget, whereas GenPO++ rapidly learns a more effective policy.
Figure~\ref{fig:real_world_episode_rewards} shows that GenPO++ achieves faster
reward improvement and higher final performance than PPO during simulation
training. Figure~\ref{fig:real_world_nutbolt} further demonstrates successful
real-world deployment, where the learned policy loosens nuts from bolts with
different geometries. These results suggest that GenPO++ improves both training
efficiency and robustness under sim-to-real variations. More details are provided
in Appendix~\ref{realworld}.

\begin{figure}
    \centering
    \includegraphics[width=1.0\linewidth]{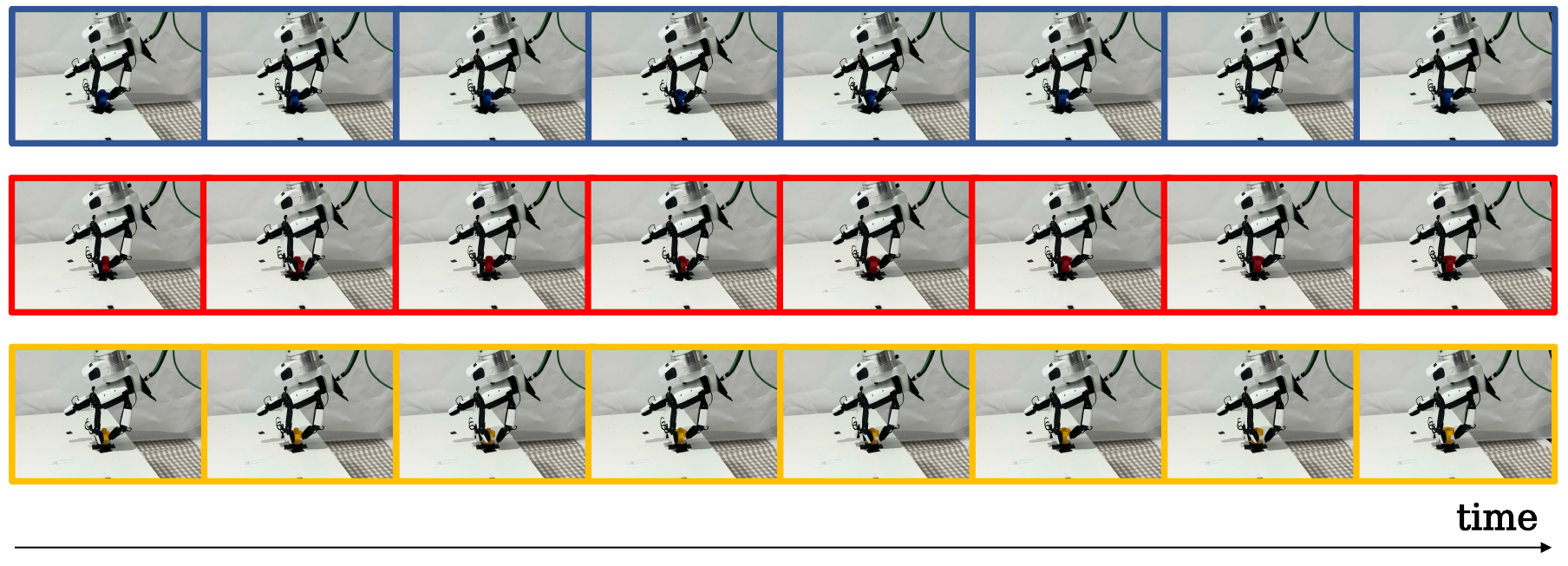}
    \caption{Sequential video frames of the real-world evaluation task. GenPO++ controls the
dexterous hand to rotate and loosen nuts from bolts of different sizes, where
different colors indicate different bolt geometries.}
\vspace{-15pt}
    \label{fig:real_world_nutbolt}
\end{figure}

% Figure~\ref{fig:real_world_nutbolt} shows a representative rollout on the real-world hardware. Compared to PPO in the same training pipeline, GenPO++ shows substantially stronger real-world generalization. The learned policy can adapt to bolts with different geometries and successfully loosen nuts that differ from the simulated configuration. This result suggests that the expressive flow policy and exact likelihood-ratio update of GenPO++ improve robustness under sim-to-real discrepancies, providing a practical benefit beyond the simulation benchmarks.
\vspace{-5pt}
\vspace{-5pt}
\section{Conclusion and Limitations}
\vspace{-5pt}
In this work, we presented \textbf{GenPO++}, a reversible flow policy optimization framework that replaces dummy-action augmentation with solver-history states, yielding exact inversion and a Jacobian-free log-determinant while preserving the original action dimension. Experiments across simulated control, imitation-to-RL fine-tuning, and real-world robotic manipulation show that this design improves stability and efficiency over approximate-ratio methods and avoids the usage of dummy actions. However, GenPO++ introduces the solver-history coefficient $\sigma$: when $\sigma$ is too large, the reversible update may deviate too much from the original flow process and distort the generated policy; when $\sigma$ is too small, numerical stability and reversibility can degrade. How to set $\sigma$ adaptively, or design improved reversible solvers that reduce sensitivity to this parameter, is an important problem for future work.

\clearpage
\bibliographystyle{plain}
\bibliography{reference}
\newpage
\appendix
\section{Algorithm}
\begin{algorithm*}[ht!]
    \caption{GenPO++}
    \label{alg:ripo}
    \textbf{Input:} flow policy $v_{\theta}(x,s,t)$ with base density $\tilde{p}_0(z)$,
    value network $V_{\omega}(s)$, solver steps $M$, rollout horizon $N$, PPO epochs $K$,
    history coefficient $\sigma$.
    \begin{algorithmic}[1]
        \State $\theta_{\mathrm{old}} \gets \theta$
        \For{$t$ \textbf{in} $1,2,\cdots,T$}
            \For{each actor}
                \State Initialize an empty rollout buffer $\mathcal{D}$
                \For{$n$ \textbf{in} $1,2,\cdots,N$}
                    \State Sample base action $z_{0,n} \sim \tilde{p}_0(z)$ and generate the augmented action $z_{1,n}$ by (\ref{eq:GenPO++_update}):
                    \State Execute the environment action $a_n=x_{1,n}$
                    \State Observe reward $r_n$ and next state $s_{n+1}$
                    \State Store $(s_n,a_n,z_{1,n},r_n,s_{n+1})$ in $\mathcal{D}$
                \EndFor
                \State Compute old augmented log-likelihoods
                $\log \tilde{\pi}_{\theta_{\mathrm{old}}}(x_1,x_{1-\Delta_t}\mid s)$ by exact inversion
                \State Estimate advantages $\hat{A}_1,\cdots,\hat{A}_N$ and returns
                $\hat{R}_1,\cdots,\hat{R}_N$ using GAE
            \EndFor

            \For{$k$ \textbf{in} $1,2,\cdots,K$}
                \State Sample a mini-batch $\mathcal{B}$ from the rollout buffer $\mathcal{D}$
                \State For each $(s,z_1)\in\mathcal{B}$, recover the base solver state under the current policy:
                \[
                z_{0,\theta}
                =
                F^{-1}_{0,\theta}
                \circ
                \cdots
                \circ
                F^{-1}_{1,\theta}
                (z_1|s),
                \]
                \State Compute the exact augmented likelihood ratio:
                \[
                \tilde{r}_{\theta}
                =
                \exp(
                \log \tilde{\pi}_{\theta}(z_{0,\theta}| s)
                -
                \log \tilde{\pi}_{\mathrm{old}}(z_{0,\rm old}| s)
                )
                \]
                \State Update $\pi_{\theta}$ by maximizing the clipped GenPO++ objective (\ref{eq:GenPO++_fiber_style}):
                \State Update the value network $V_{\omega}(s)$ with
                \[
                \mathcal{L}_{V}(\omega)
                =
                \mathbb{E}_{\mathcal{B}}
                [
                \|
                V_{\omega}(s)-\hat{R}
                \|^2
                ]\]
                
            \EndFor
            \State $\theta_{\mathrm{old}} \gets \theta$
        \EndFor
    \end{algorithmic}
\end{algorithm*}

\section{ELBO-Ratio Bias in FPO}
\label{app:elbo_ratio_bias}

FPO replaces the exact likelihood ratio by an ELBO ratio. For a fixed
state-action pair \((s,a)\), write the ELBO as
\begin{equation}
    \mathcal{E}_\theta(s,a)
    =
    \log \pi_\theta(a|s)-\Delta_\theta(s,a),
    \label{eq:app_elbo_decomp}
\end{equation}
where \(\Delta_\theta(s,a)\ge 0\) denotes the variational gap between the
true log-likelihood and the ELBO.

\begin{proposition}[ELBO-ratio decomposition]
\label{prop:app_elbo_ratio_bias}
Let
\begin{equation}
   \rho_{\rm FPO}(\theta;s,a)
    =
    \exp\!(
        \mathcal{E}_\theta(s,a)
        -
        \mathcal{E}_{\theta_{\rm old}}(s,a)
    )
    \label{eq:app_elbo_ratio}
\end{equation}
be the ELBO ratio used in place of the exact policy ratio. Then
\begin{equation}
   \rho_{\rm FPO}(\theta;s,a)
    =
    r_\pi(\theta;s,a)
    \exp\!(
        \Delta_{\theta_{\rm old}}(s,a)
        -
        \Delta_\theta(s,a)
    ).
    \label{eq:app_elbo_ratio_decomp}
\end{equation}
Moreover, for the unclipped local objective
\begin{equation}
    \mathcal{L}_{\rm FPO}(\theta;s,a)
    =
    A(s,a)\rho_{\rm E}(\theta;s,a),
\end{equation}
we have
\begin{equation}
    \nabla_\theta \mathcal{L}_{\rm FPO}
    =
    A(s,a)\rho_{\rm E}(\theta;s,a)
    [
        \nabla_\theta \log \pi_\theta(a|s)
        -
        \nabla_\theta \Delta_\theta(s,a)
    ].
    \label{eq:app_elbo_ratio_grad}
\end{equation}
\end{proposition}

\begin{proof}
Substituting Eq.~\eqref{eq:app_elbo_decomp} into
Eq.~\eqref{eq:app_elbo_ratio} gives
\begin{align}
   \rho_{\rm FPO}(\theta;s,a)
    &=
    \exp\!(
        \log \pi_\theta(a|s)
        -
        \Delta_\theta(s,a)
        -
        \log \pi_{\theta_{\rm old}}(a|s)
        +
        \Delta_{\theta_{\rm old}}(s,a)
    ) \notag \\
    &=
    \frac{\pi_\theta(a|s)}
         {\pi_{\theta_{\rm old}}(a|s)}
    \exp\!(
        \Delta_{\theta_{\rm old}}(s,a)
        -
        \Delta_\theta(s,a)
    ).
\end{align}
Since
\(\theta_{\rm old}\), \(A(s,a)\), and
\(\mathcal{E}_{\theta_{\rm old}}(s,a)\) are fixed during the update,
\begin{align}
    \nabla_\theta \mathcal{L}_{\rm FPO}
    &=
    A(s,a)\rho_{\rm E}(\theta;s,a)
    \nabla_\theta \mathcal{E}_\theta(s,a) \notag \\
    &=
    A(s,a)\rho_{\rm E}(\theta;s,a)
    [
        \nabla_\theta \log \pi_\theta(a|s)
        -
        \nabla_\theta \Delta_\theta(s,a)
    ],
\end{align}
\end{proof}
Eq.~\eqref{eq:app_elbo_ratio_grad} shows that the ELBO-ratio gradient is not
only a likelihood-ratio gradient. Its deviation from the likelihood direction is
\begin{equation}
    -A(s,a)\rho_{\rm E}(\theta;s,a)
    \nabla_\theta \Delta_\theta(s,a).
    \label{eq:app_gap_gradient}
\end{equation}
For \(A(s,a)<0\), this term has a positive coefficient in the direction of
\(\nabla_\theta \Delta_\theta(s,a)\). In the degenerate case
\(\nabla_\theta \log \pi_\theta(a|s)=0\),
\begin{equation}
    \nabla_\theta \mathcal{L}_{\rm E}
    =
    -A(s,a)\rho_{\rm E}(\theta;s,a)
    \nabla_\theta \Delta_\theta(s,a),
\end{equation}
so the update changes only the ELBO gap while leaving the executed-action
likelihood locally unchanged. Thus, when the variational gap varies across
policy updates, the ELBO ratio can differ from the true policy ratio and can
induce clipped PPO-style updates that are not aligned with the true
executed-action likelihood.

\section{Proof of Proposition~\ref{prop:high_order_history}}
\label{AB2_GENPO++}
% \begin{proof}
Recall the two-step Adams--Bashforth update and rewrite as:
\begin{equation}
    x_{i+1}^{\mathrm{AB2}}
    =
    \underbrace{x_i+\Delta_t\,v_i}_{\text{Euler step}}
    -
    \underbrace{
    \frac{1}{2}\Delta_t
    (
        v_{i-1}-v_i
    )
    }_{\text{derivative-history correction}} .
\end{equation}

Let \(x(t)\) be the smooth trajectory of the ODE
\begin{equation}
     \frac{dx(t)}{dt}=v_\theta(x(t),s,t),
\end{equation}
   
and let $ x_i=x(t_i)$, $x_{i-1}=x(t_{i-1})$.
This indexing convention matches the reverse solver direction: \(x_{i-1}\)
is the history state one step before \(x_i\), while \(x_{i+1}\) is the next state to be predicted.

Define the total derivative of the velocity field along the trajectory as
\begin{equation}
    \dot v_i
    =
    \frac{d}{dt}
    v_\theta(x(t),t)
    |_{t=t_i}.
\end{equation}
By Taylor expansion of \(v_\theta(x(t),s,t)\) around \(t_i\), we have
\begin{equation}
    v_{i-1}= v_i-\Delta_t\,\dot v_i+O(\Delta_t^2).
\end{equation}
Hence the derivative-history correction is:
\begin{equation}
     -\frac{1}{2}\Delta_t(v_{i-1}-v_i)=
    \frac{1}{2}\Delta_t(\Delta_t\,\dot v_i+O(\Delta_t^2)) 
    =\frac{1}{2}\Delta_t^2\,\dot v_i+O(\Delta_t^3).
\end{equation}

On the other hand, Taylor expansion of the state trajectory gives
\begin{equation}
\begin{aligned}
    x_{i-1}
    &=x_i-\Delta_t\,\dot x(t_i)+\frac{1}{2}\Delta_t^2\,\ddot x(t_i)+O(\Delta_t^3)\\
    &=x_i-\Delta_t\,v_i+\frac{1}{2}\Delta_t^2\,\dot v_i+O(\Delta_t^3).
\end{aligned}
\end{equation}
Therefore,
\begin{equation}
  -\frac{1}{2}\Delta_t (v_{i-1}-v_i)=x_{i-1}-x_i+\Delta_t\,v_i+O(\Delta_t^3).  
\end{equation}
    
GenPO++ replaces the history correction with a relaxed state-history residual, which gives:
\begin{equation}
\begin{aligned}
    x_{i+1}^{\mathrm{GenPO++}}
    &=
    x_i
    +
    \Delta_t\,v_i
    +
    \sigma
    (
        x_{i-1}-x_i+\Delta_t\,v_i
    ) \\
    &=
    x_i
    +
    \Delta_t\,v_i
    +
    \sigma x_{i-1}
    -
    \sigma x_i
    +
    \sigma\Delta_t\,v_i \\
    &=
    (1-\sigma)x_i
    +
    \sigma x_{i-1}
    +
    (1+\sigma)\Delta_t\,v_i.
\end{aligned}
\end{equation}
which is Eq.~\eqref{eq:GenPO++_update}.

When \(\sigma=1\), the GenPO++ update becomes
\begin{equation}
     x_{i-1}^{\mathrm{GenPO++}}
    =
    x_i
    +
    \Delta_t\,v_i
    +
    x_{i+1}-x_i+\Delta_t\,v_i,
\end{equation}
which matches the AB2 update up to the local approximation error
\(O(\Delta_t^3)\). For general \(\sigma\), the coefficient \(\sigma\) relaxes
the strength of the state-history correction while preserving the explicit
dependence on \((x_i,x_{i+1})\), which is the key property enabling the
bidirectionally explicit GenPO++ transition.
\begin{figure}
    \centering
    \includegraphics[width=0.7\linewidth]{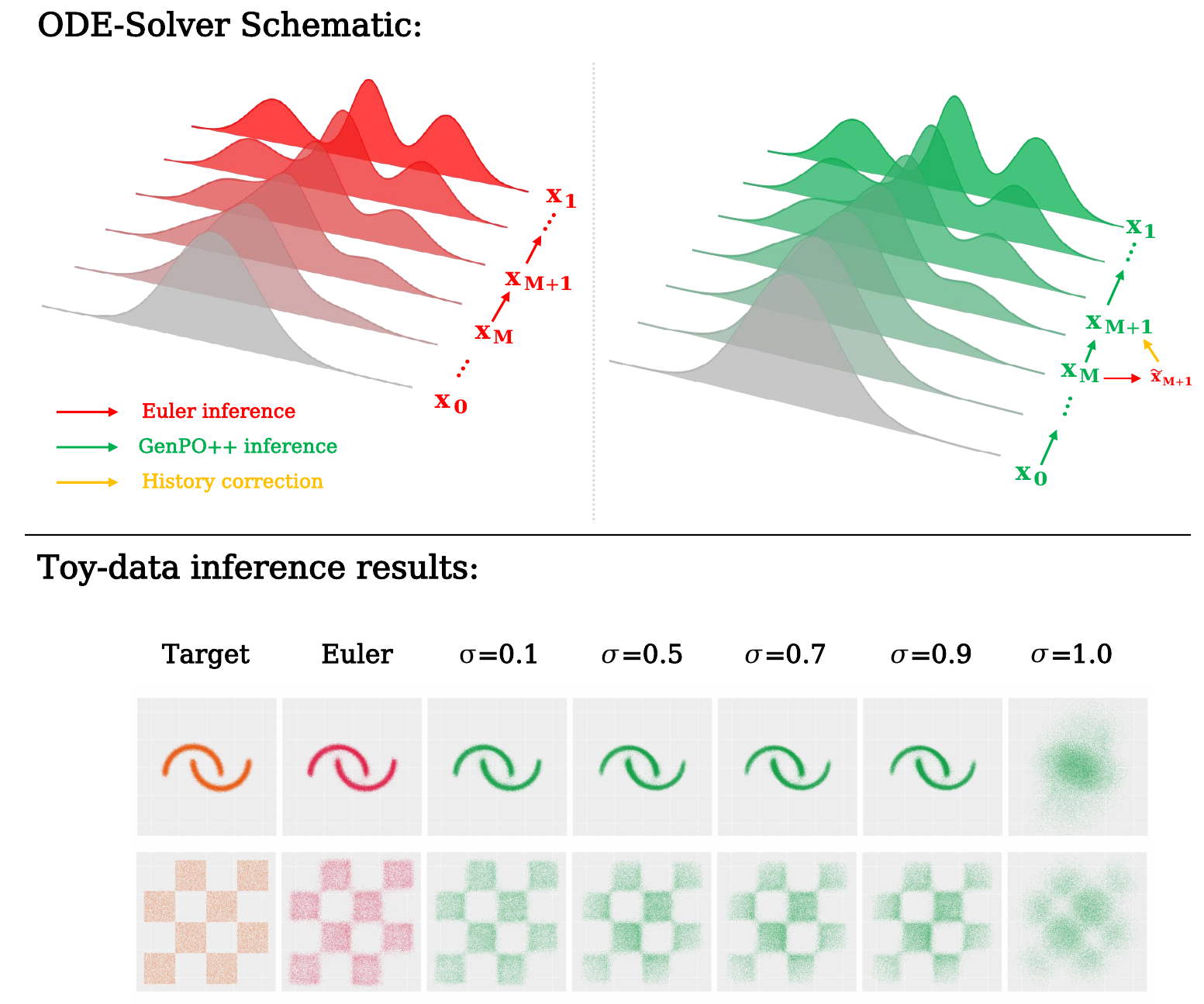}
    \caption{We train flow-matching model on toy data and compare Euler sampling with GenPO++ sampling under different $\sigma$.}
    \label{fig:schematic}
\end{figure}
\section{Proof of Theorem~\ref{thm:GenPO++_lte}}
\label{LTE}
Define the second-order derivative of the velocity field along the trajectory as:
\begin{equation}
    \ddot v_i
    =
    \frac{d^2}{dt^2}
    v_\theta(x(t),t|s)\bigg|_{t=t_i}.
\end{equation}
Taylor expansion gives
\begin{equation}
\begin{aligned}
    x_{i-1}
    &=
    x_i
    -\Delta t\,v_i
    +\frac{1}{2}\Delta t^2\dot v_i
    -\frac{1}{6}\Delta t^3\ddot v_i
    +O(\Delta t^4), \\
    x_{i+1}^{\star}
    &=
    x_i
    +\Delta t\,v_i
    +\frac{1}{2}\Delta t^2\dot v_i
    +\frac{1}{6}\Delta t^3\ddot v_i
    +O(\Delta t^4).
\end{aligned}
\end{equation}

Substituting the exact history pair $(x_i,x_{i-1})$ into Eq.~(\ref{eq:GenPO++_update}), the GenPO++ update gives
\begin{equation}
\begin{aligned}
    \hat x_{i+1}
    &=
    (1-\sigma)x_i+\sigma x_{i-1}
    +(1+\sigma)\Delta t\,v_i \\
    &=
    (1-\sigma)x_i
    +\sigma
    \left(
        x_i
        -\Delta t\,v_i
        +\frac{1}{2}\Delta t^2\dot v_i
        -\frac{1}{6}\Delta t^3\ddot v_i
        +O(\Delta t^4)
    \right)
    +(1+\sigma)\Delta t\,v_i \\
    &=
    x_i
    +\Delta t\,v_i
    +\frac{\sigma}{2}\Delta t^2\dot v_i
    -\frac{\sigma}{6}\Delta t^3\ddot v_i
    +O(\Delta t^4).
\end{aligned}
\end{equation}
Therefore,
\begin{equation}
\begin{aligned}
    \hat x_{i+1}-x_{i+1}^{\star}
    &=
    \left(
        x_i
        +\Delta t\,v_i
        +\frac{\sigma}{2}\Delta t^2\dot v_i
        -\frac{\sigma}{6}\Delta t^3\ddot v_i
        +O(\Delta t^4)
    \right) \\
    &\quad -
    \left(
        x_i
        +\Delta t\,v_i
        +\frac{1}{2}\Delta t^2\dot v_i
        +\frac{1}{6}\Delta t^3\ddot v_i
        +O(\Delta t^4)
    \right) \\
    &=
    \frac{\sigma-1}{2}\Delta t^2\dot v_i
    -
    \frac{\sigma+1}{6}\Delta t^3\ddot v_i
    +O(\Delta t^4).
\end{aligned}
\end{equation}
Thus,
\begin{equation}
    \hat x_{i+1}-x_{i+1}^{\star}
    =
    \frac{\sigma-1}{2}\Delta t^2
    \frac{d}{dt}v_\theta(x(t),t|s)\bigg|_{t=t_i}
    +O(\Delta t^3).
\end{equation}
This $\sigma$-dependent local error explains the role of the history coefficient. Figure~\ref{fig:schematic} visualizes this effect on
toy-data sampling with the same trained flow-matching model.
\section{Conservative Policy Improvement with Augmented-Action Ratios}
\label{app:augmented_ratio_bound}

\begin{theorem}[Consistency of GenPO++ augmented-action ratios]
The executed-action likelihood ratio satisfies
\begin{equation}
    r_\theta(x_1|s)
=
\mathbb E_{x_{1-\Delta_t}\sim \tilde\pi_{\rm old}(\cdot|x_1,s)}
[
\tilde r_\theta(x_1,x_{1-\Delta_t}|s)
].
\label{ratio}
\end{equation}

Moreover, for advantage function
\(A^{\pi_{\rm old}}(s,x_1)\),
\begin{equation}
    \mathbb E_{s\sim d_\rho^{\pi_{\rm old}},
x_1\sim \pi_{\rm old}(\cdot|s)}
[
r_\theta(x_1|s)A^{\pi_{\rm old}}(s,x_1)
]
=
\mathbb E_{s\sim d_\rho^{\pi_{\rm old}},
(z_1)\sim \tilde\pi_{\rm old}(\cdot,\cdot|s)}
[
\tilde r_\theta(z_1|s)A^{\pi_{\rm old}}(s,x_1)
].
\label{objective}
\end{equation}
\end{theorem}
\begin{proof}
$z_i$ is defined as $(x_i, x_{i-\Delta_t})$. We first prove Eq.(\ref{ratio}). Fix a state \(s\) and an executed action \(x_1\).
By the definition of conditional density,
\begin{equation}
\tilde\pi_{\rm old}(x_{1-\Delta_t}|x_1,s)
=
\frac{\tilde\pi_{\rm old}(x_1,x_{1-\Delta_t}|s)}
{\pi_{\rm old}(x_1|s)}.
\end{equation}
Therefore,
\begin{equation}
\begin{aligned}
&\mathbb E_{x_{1-\Delta_t}\sim \tilde\pi_{\rm old}(\cdot|x_1,s)}
[
\tilde r_\theta(x_1,x_{1-\Delta_t}|s)
] \\&=
\int
\tilde r_\theta(x_1,x_{1-\Delta_t}|s)
\tilde\pi_{\rm old}(x_{1-\Delta_t}|x_1,s)
\,dx_{1-\Delta_t}\\
&=\int
\tilde r_\theta(x_1,x_{1-\Delta_t}|s)
\frac{\tilde\pi_{\rm old}(x_1,x_{1-\Delta_t}|s)}
{\pi_{\rm old}(x_1|s)}
\,dx_{1-\Delta_t}\\
&=
\int
\frac{\tilde\pi_\theta(x_1,x_{1-\Delta_t}|s)}
{\tilde\pi_{\rm old}(x_1,x_{1-\Delta_t}|s)}
\frac{\tilde\pi_{\rm old}(x_1,x_{1-\Delta_t}|s)}
{\pi_{\rm old}(x_1|s)}
\,dx_{1-\Delta_t} \\
&=
\frac{1}{\pi_{\rm old}(x_1|s)}
\int
\tilde\pi_\theta(x_1,x_{1-\Delta_t}|s)
\,dx_{1-\Delta_t}\\
&=\frac{\pi_\theta(x_1|s)}
{\pi_{\rm old}(x_1|s)}=
r_\theta(x_1|s).
\end{aligned}
\end{equation}

We now prove Eq.(\ref{objective}). Fix a state \(s\). Starting from the standard
executed-action surrogate, we have
\begin{equation}
\begin{aligned}
&\mathbb E_{x_1\sim \pi_{\rm old}(\cdot|s)}
[
r_\theta(x_1|s)A^{\pi_{\rm old}}(s,x_1)
] \\
&=
\int
r_\theta(x_1|s)
A^{\pi_{\rm old}}(s,x_1)
\pi_{\rm old}(x_1|s)
\,dx_1\\
&=
\int\int
\tilde r_\theta(x_1,x_{1-\Delta_t}|s)
\tilde\pi_{\rm old}(x_{1-\Delta_t}|x_1,s)
\,dx_{1-\Delta_t}
A^{\pi_{\rm old}}(s,x_1)
\pi_{\rm old}(x_1|s)
\,dx_1.
\end{aligned}
\end{equation}

With the bayes's rules, $\tilde\pi_{\rm old}(x_1,x_{1-\Delta_t}|s)
=
\tilde\pi_{\rm old}(x_{1-\Delta_t}|x_1,s)
\pi_{\rm old}(x_1|s)$, we obtain
\begin{equation}
\begin{aligned}
&\int
\int
\tilde r_\theta(x_1,x_{1-\Delta_t}|s)
A^{\pi_{\rm old}}(s,x_1)
\tilde\pi_{\rm old}(x_{1-\Delta_t}|x_1,s)
\pi_{\rm old}(x_1|s)
\,dx_{1-\Delta_t}\,dx_1 \\
&=
\int
\int
\tilde r_\theta(x_1,x_{1-\Delta_t}|s)
A^{\pi_{\rm old}}(s,x_1)
\tilde\pi_{\rm old}(x_1,x_{1-\Delta_t}|s)
\,dx_{1-\Delta_t}\,dx_1\\&=\mathbb E_{(x_1,x_{1-\Delta_t})\sim \tilde\pi_{\rm old}(\cdot,\cdot|s)}
[
\tilde r_\theta(x_1,x_{1-\Delta_t}|s)
A^{\pi_{\rm old}}(s,x_1)
].
\end{aligned}
\end{equation}
Therefore, for every fixed state \(s\),
\begin{equation}
\mathbb E_{x_1\sim \pi_{\rm old}(\cdot|s)}
[
r_\theta(x_1|s)A^{\pi_{\rm old}}(s,x_1)
]
=
\mathbb E_{(x_1,x_{1-\Delta_t})\sim \tilde\pi_{\rm old}(\cdot,\cdot|s)}
[
\tilde r_\theta(x_1,x_{1-\Delta_t}|s)
A^{\pi_{\rm old}}(s,x_1)
].
\end{equation}
Finally, integrating both sides over
\(s\sim d_\rho^{\pi_{\rm old}}\) gives
\begin{equation}
\mathbb E_{s\sim d_\rho^{\pi_{\rm old}},
x_1\sim \pi_{\rm old}(\cdot|s)}
[
r_\theta(x_1|s)A^{\pi_{\rm old}}(s,x_1)
]
=
\mathbb E_{s\sim d_\rho^{\pi_{\rm old}},
(z_1)\sim \tilde\pi_{\rm old}(\cdot,\cdot|s)}
[
\tilde r_\theta(z_1|s)A^{\pi_{\rm old}}(s,x_1)
].
\end{equation}
This proves Eq.(\ref{objective}).
\end{proof}

\begin{lemma}[Marginal TV is bounded by augmented TV]
For every state \(s\),
\begin{equation}
{\rm TV}\!(
\pi_\theta(\cdot|s),
\pi_{\rm old}(\cdot|s)
)
\le
{\rm TV}\!(
\tilde\pi_\theta(\cdot,\cdot|s),
\tilde\pi_{\rm old}(\cdot,\cdot|s)
).
\label{tvbound}
\end{equation}
\end{lemma}

\begin{proof}
By definition of total variation distance,
\begin{equation}
{\rm TV}\!(
\pi_\theta(\cdot|s),
\pi_{\rm old}(\cdot|s)
)
=
\frac12
\int
|
\pi_\theta(x_1|s)-\pi_{\rm old}(x_1|s)
|
dx_1.
\end{equation}
Using the marginal definitions,
\begin{equation}
\pi_\theta(x_1|s)-\pi_{\rm old}(x_1|s)
=
\int
[
\tilde\pi_\theta(x_1,x_{1-\Delta_t}|s)
-
\tilde\pi_{\rm old}(x_1,x_{1-\Delta_t}|s)
]
dx_{1-\Delta_t}.
\end{equation}
Therefore,
\begin{equation}
\begin{aligned}
&{\rm TV}\!(
\pi_\theta(\cdot|s),
\pi_{\rm old}(\cdot|s)
) 
=
\frac12
\int
|
\int
[
\tilde\pi_\theta(x_1,x_{1-\Delta_t}|s)
-
\tilde\pi_{\rm old}(x_1,x_{1-\Delta_t}|s)
]
dx_{1-\Delta_t}
|
dx_1.
\end{aligned}
\end{equation}
By the triangle inequality,
\begin{equation}
|
\int
[
\tilde\pi_\theta(x_1,x_{1-\Delta_t}|s)
-
\tilde\pi_{\rm old}(x_1,x_{1-\Delta_t}|s)
]
dx_{1-\Delta_t}
|
\le
\int
|
\tilde\pi_\theta(x_1,x_{1-\Delta_t}|s)
-
\tilde\pi_{\rm old}(x_1,x_{1-\Delta_t}|s)
|
dx_{1-\Delta_t}.
\end{equation}
Thus,
\begin{equation}
\begin{aligned}
{\rm TV}\!(
\pi_\theta(\cdot|s),
\pi_{\rm old}(\cdot|s)
) 
&\le
\frac12
\int\int
|
\tilde\pi_\theta(x_1,x_{1-\Delta_t}|s)
-
\tilde\pi_{\rm old}(x_1,x_{1-\Delta_t}|s)
|
dx_{1-\Delta_t}dx_1 \\
&=
{\rm TV}\!(
\tilde\pi_\theta(\cdot,\cdot|s),
\tilde\pi_{\rm old}(\cdot,\cdot|s)
).
\end{aligned}
\end{equation}
\end{proof}

According to the Eq.(\ref{objective}) and Eq.(\ref{tvbound}), there is:
\begin{equation}
\begin{aligned}
\mathcal{J}(\pi_\theta)-\mathcal{J}(\pi_{\rm old}) &\geq \mathbb E_{s\sim d_\rho^{\pi_{\rm old}},
x_1\sim \pi_{\rm old}(\cdot|s)}[
r_\theta(x_1|s)A^{\pi_{\rm old}}(s,x_1)]-C_{\rm TV}{\rm TV}\!(
\tilde\pi_\theta(\cdot|s),
\tilde\pi_{\rm old}(\cdot|s)
)\\
&=\mathbb E_{s\sim d_\rho^{\pi_{\rm old}},
z_1\sim \tilde\pi_{\rm old}(\cdot,\cdot|s)}
[
\tilde r_\theta(z_1|s)A^{\pi_{\rm old}}(s,x_1)
]-C_{\rm TV}{\rm TV}\!(
\tilde\pi_\theta(\cdot|s),
\tilde\pi_{\rm old}(\cdot|s)
)\\
&\geq \mathbb E_{s\sim d_\rho^{\pi_{\rm old}},z_1\sim \tilde\pi_{\rm old}(\cdot,\cdot|s)}
[
\tilde r_\theta(z_1|s)A^{\pi_{\rm old}}(s,x_1)
]-C_{\rm TV}{\rm TV}\!(
\tilde\pi_\theta(\cdot, \cdot|s),
\tilde\pi_{\rm old}(\cdot, \cdot|s)
),
\end{aligned}
\end{equation}
Therefore, applying PPO clipping to \(\tilde r_\theta(x_1,x_{1-\Delta_t}|s)\) gives a conservative trust-region surrogate for the original MDP, even though the likelihood ratio is computed in augmented
GenPO++ coordinates.

\section{Experiment Details}
\subsection{Hardware}
Experiments were conducted on a dual-socket server featuring Intel Xeon Gold 6430 CPUs (128 total threads, 2.1–3.4 GHz) and 8 NVIDIA RTX 4090 D GPUs (24GB GDDR6X each). The system utilized CUDA 12.8 with driver version 570.124, organized across two NUMA nodes. No Multi-Instance GPU (MIG) or ECC features were enabled.
\subsection{IsaacLab}
\paragraph{Reinforcement Learning Framework in IsaacLab.}
All reinforcement learning algorithms are implemented within IsaacLab using the RSL-RL training framework. IsaacLab provides official wrappers that convert environment observations, actions, rewards, and termination signals into the interface required by different RL libraries, including RSL-RL. In our implementation, we build all policy optimization methods on top of the RSL-RL
on-policy runner and keep the same environment interaction pipeline, rollout storage, mini-batch construction, and PPO-style optimization interface across methods. This unified implementation ensures that the reported differences come from the policy parameterization and likelihood-ratio computation rather than from differences in the simulator or training infrastructure. The RSL-RL library is available at
\url{https://github.com/leggedrobotics/rsl_rl}. The FPO is available at \url{https://github.com/amazon-far/fpo-control}. The PolicyFlow is available at \url{https://github.com/PolicyFlow2026/PolicyFlow}. The GenPO is available at \url{https://github.com/wadx2019/genpo/}. 
\begin{figure}
    \centering
    \includegraphics[width=0.5\linewidth]{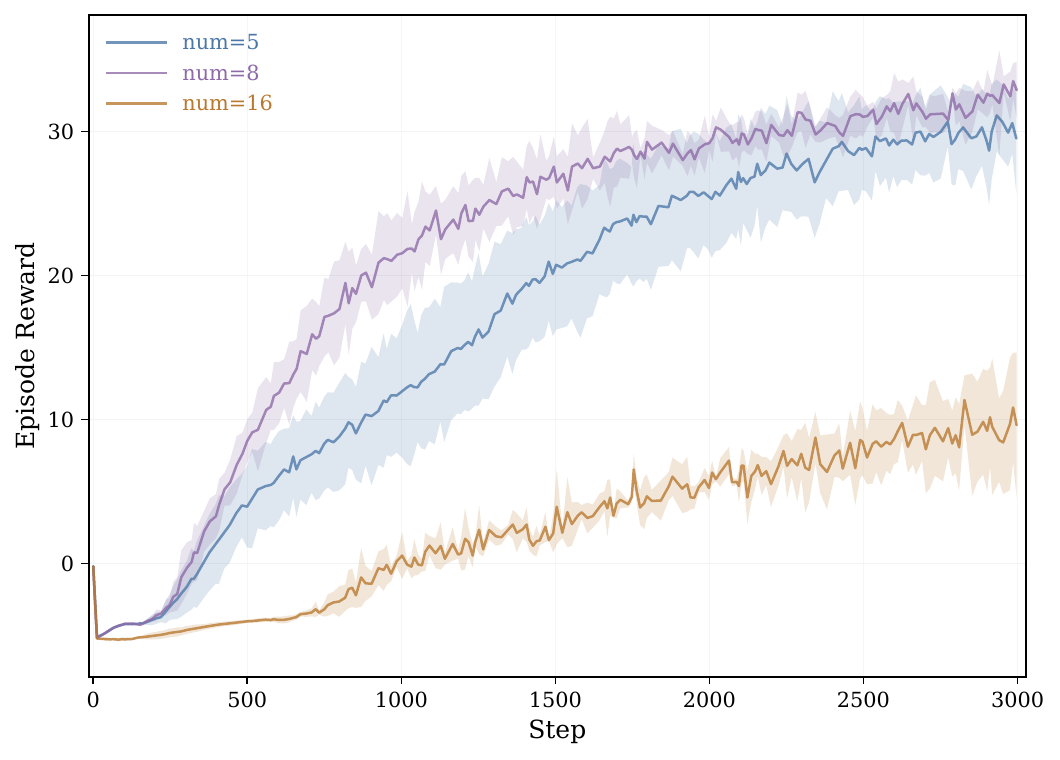}
    \caption{Learning curves for different flow policy time steps on the Isaaclab-Vecocity-Rough-G1-v0 benchmark. Results are averaged over 5 runs. The x-axis denotes training epochs, and the y-axis shows average episodic return with one standard deviation shaded.}
    \label{fig:ablationstep}
\end{figure}
\begin{figure}
    \centering
    \includegraphics[width=1.0\linewidth]{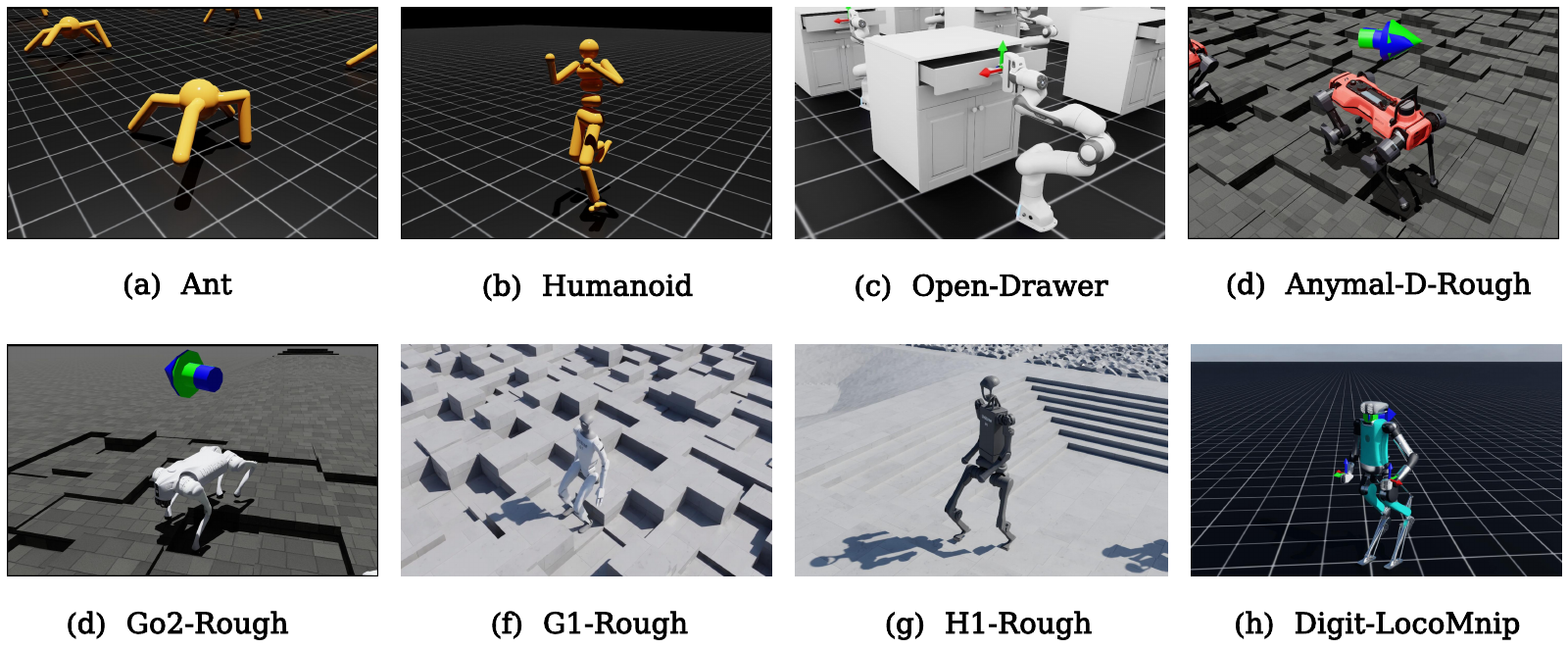}
    \caption{Eight Isaaclab benchmark visualizations, eight images from~\url{https://isaac-sim.github.io/IsaacLab/main/source/overview/environments.html}. From (a) to (h) are  Isaac-Ant-v0,  Isaac-Humanoid-v0, Isaac-Open-Drawer-Franka-v0, Isaac-Velocity-Rough-Anymal-D-v0, Isaac-Velocity-Rough-Unitree-Go2-v0, Isaac-Velocity-Rough-G1-v0, Isaac-Velocity-Rough-H1-v0, and Isaac-Tracking-LocoManip-Digit-v0. }
    \label{fig:isaaclabbench}
\end{figure}

\begin{figure}
    \centering
    \includegraphics[width=1.0\linewidth]{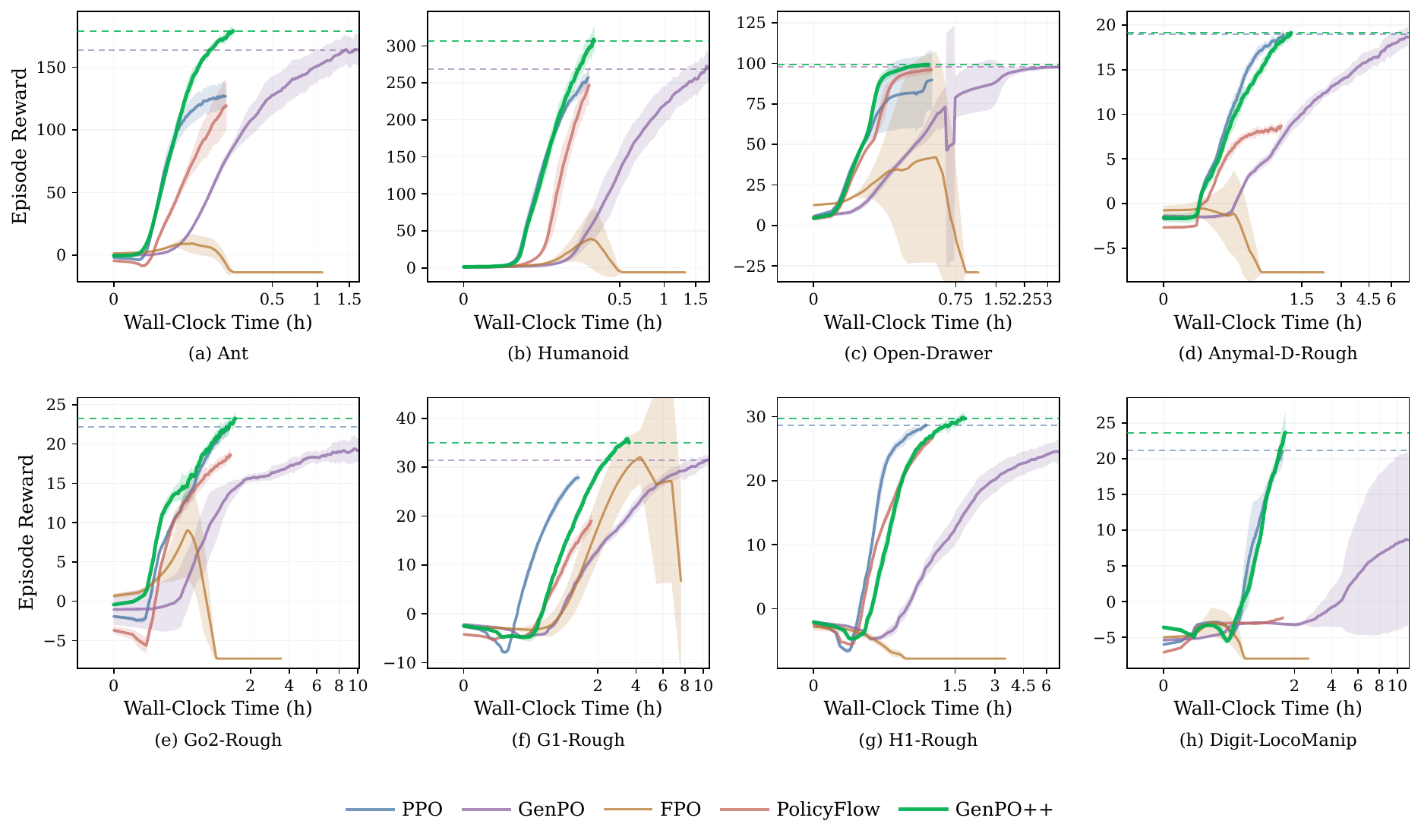}
    \caption{Learning curves across 8 IsaacLab benchmarks. Results are averaged over 5 runs. The x-axis denotes wall-clock-time, and the y-axis shows average episodic return with one standard deviation shaded.}
    \label{fig:isaaclabtime}
\end{figure}

\paragraph{Benchmark suite.}
We evaluate all methods on eight IsaacLab tasks covering three representative categories:
classical locomotion, articulated manipulation, and whole-body locomotion-control. The benchmark
contains \textsc{Isaac-Ant-v0}, \textsc{Isaac-Humanoid-v0},
\textsc{Isaac-Open-Drawer-Franka-v0}, \textsc{Isaac-Velocity-Rough-Anymal-D-v0},
\textsc{Isaac-Velocity-Rough-Unitree-Go2-v0},
\textsc{Isaac-Velocity-Rough-Unitree-G1-v0},
\textsc{Isaac-Velocity-Rough-H1-v0}, and
\textsc{Isaac-Tracking-LocoManip-Digit-v0}. These tasks differ substantially in action dimension,
contact complexity, reward scale, and episode length. We keep the official IsaacLab environment
configuration, reward terms, termination conditions, and simulation settings unchanged for all
methods. Therefore, each method is evaluated under the same control problem and the same simulator
dynamics.

\paragraph{Baselines and fairness.}
We compare GenPO++ with Gaussian PPO and three representative generative policy optimization methods: FPO, GenPO, and PolicyFlow. For all methods,
the critic architecture, rollout horizon, advantage estimation, value loss, PPO clipping range, and
mini-batch construction follow the same RSL-RL pipeline unless otherwise specified. For FPO, we follow the parameter design in the code, the number of epochs is 32, and an EMA mechanism is introduced\footnote{\url{https://github.com/amazon-far/fpo-control}}.

\paragraph{Wall-clock comparison.}
Figure~\ref{fig:isaaclabtime} reports the IsaacLab learning curves using wall-clock time as the
x-axis. This complements the iteration-based curves in the main paper. GenPO++ remains competitive
or superior under wall-clock comparison, indicating that the performance gain is not caused by using
more expensive updates. In contrast, GenPO incurs a larger learning overhead due to exact likelihood
evaluation through the augmented inverse process, while FPO can be slower because of its heavier
optimization schedule. GenPO++ avoids neural-Jacobian computation and therefore substantially
reduces policy-update time while retaining exact augmented likelihood-ratio optimization.

\paragraph{Number of flow steps.}
Figure~\ref{fig:ablationstep} studies the number of flow-policy solver steps on
\textsc{Isaac-Velocity-Rough-G1-v0}. Although more steps increase the expressiveness of the
transport map, they also lengthen the differentiable inversion path used for likelihood-ratio
computation. During PPO updates, gradients are back-propagated through the unrolled reversible
solver, making the optimization similar to BPTT. A larger number of steps therefore increases
memory and computation cost, and may cause poorly conditioned gradients due to repeated
step-wise Jacobian products. Empirically, moderate step numbers provide the best trade-off, while
using more steps does not yield monotonic improvement.
\begin{table}[t]
\centering
\small
\caption{Task-specific hyperparameters for GenPO++.}
\label{tab:belm_genpo_task_hyperparams}
\resizebox{\textwidth}{!}{
\begin{tabular}{lccccccccc}
\toprule
Task &
Steps/env &
Iterations &
Hidden dims &
Activation &
Flow steps &
$\sigma$ &
Epochs &
LR &
KL \\
\midrule
Ant & 32 & 1000 & $[400,200,100]$ & ELU & 5 & 0.75 & 5 & $1.0{\times}10^{-3}$ & 0.01 \\
Humanoid & 32 & 1000 & $[400,200,100]$ & ELU & 5 & 0.75 & 5 & $5.0{\times}10^{-4}$ & 0.01 \\
Open-Drawer & 96 & 400 & $[256,128,64]$ & ELU & 5 & 0.75& 5 & $5.0{\times}10^{-4}$ & 0.02 \\
ANYmal-D-Rough & 24 & 1500 & $[512,256,128]$ & ELU & 5 & 0.75 & 5 & $1.0{\times}10^{-3}$ & 0.01 \\
Go2-Rough & 24 & 1500 &$[512,256,128]$ & ELU & 5 & 0.75 & 5 & $1.0{\times}10^{-3}$ & 0.01 \\
G1-Rough & 24 & 3000 &$[512,256,128]$ & ELU & 5 & 0.75 & 32 & $1.0{\times}10^{-3}$ & 0.01 \\
H1-Rough & 24 & 3000 & $[512,256,128]$ & ELU & 8 & 0.75 & 5 & $1.0{\times}10^{-3}$ & 0.01 \\
Digit-LocoManip & 24 & 2000 & $[256,128,64]$ & ELU & 5 & 0.75 & 5 & $1.0{\times}10^{-3}$ & 0.01 \\
\bottomrule
\end{tabular}
}
\end{table}

\subsection{Fine-tuning: Robomimic}
\begin{figure}
    \centering
    \includegraphics[width=1.0\linewidth]{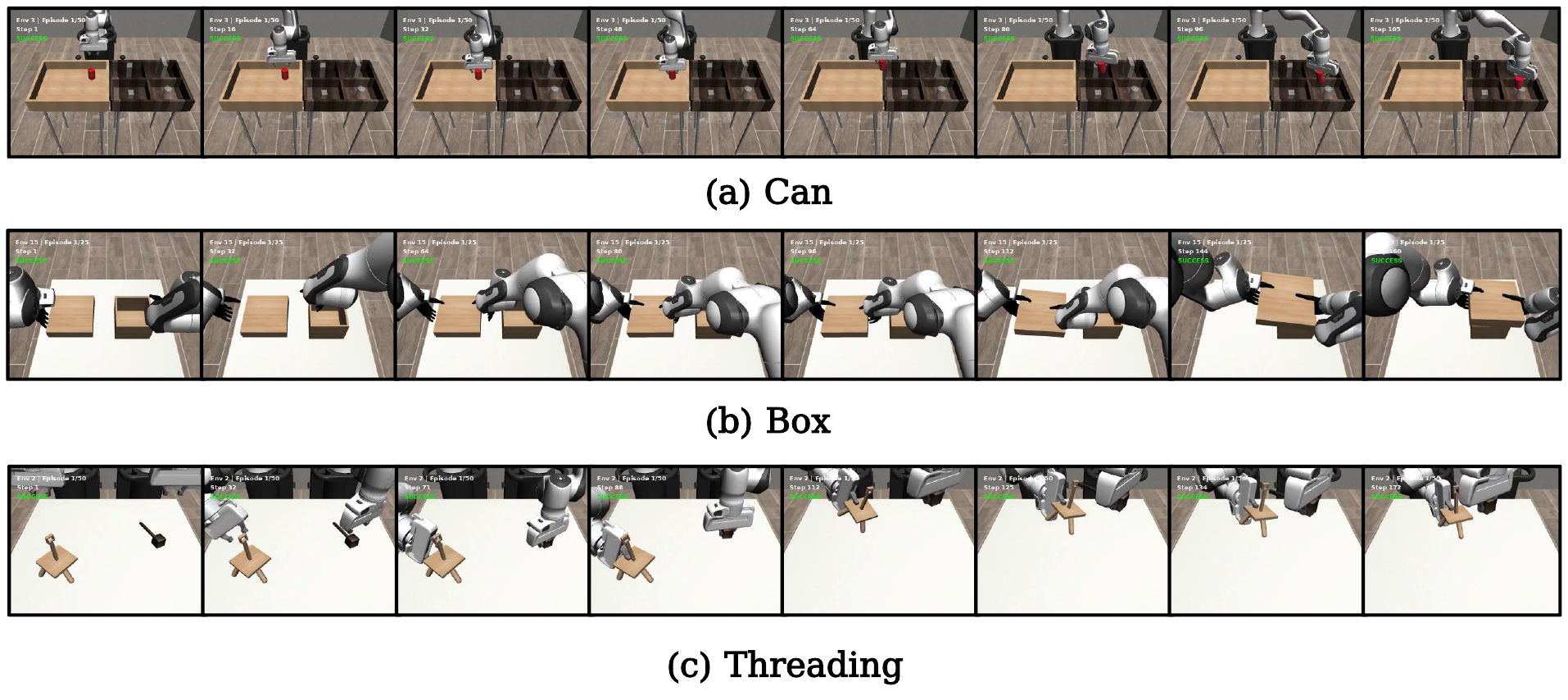}
    \caption{Evaluation episodes on the three manipulation fine-tuning tasks.}
    \label{fig:mimic}
\end{figure}
We evaluate online fine-tuning of pretrained flow-matching manipulation policies on three
visual manipulation tasks: \textsc{Can}, \textsc{Box Cleanup}, and \textsc{Threading}. The \textsc{Can} task is from RoboMimic, while \textsc{Box Cleanup} and \textsc{Threading} follow the DexMimicGen manipulation benchmark used in prior flow-policy fine-tuning experiments. These tasks cover different manipulation regimes, including single-object grasping, object rearrangement, and long-horizon precision insertion.

Our Robomimic fine-tuning experiments are implemented based on the FPO codebase and use the
same pretrained base flow policies released by FPO:
\url{https://github.com/amazon-far/fpo-control/tree/main/manipulation_experiments}. We evaluate three manipulation tasks:
\textsc{Can}, \textsc{Box Cleanup}, and \textsc{Threading}. These tasks cover different robot
embodiments and manipulation regimes: \textsc{Can} uses a single robot arm with a parallel-jaw
gripper, \textsc{Threading} uses two dexterous hands, and \textsc{Box Cleanup} uses two robot
arms with parallel-jaw grippers.

Following FPO, the policy outputs an action chunk with horizon length 16. During environment
interaction, all methods execute the full 16-step action chunk before querying the policy again.
Thus, the compared methods use the same temporal abstraction, environment rollout protocol, and
pretrained initialization. 

\subsection{Simulation Modeling}
\label{realworld}
The task uses the RobotEra Xhand with 12 actuated joints. Following DexScrew, we use a simplified screw-and-nut model with a fixed base and the nut is attached through a single revolute joint around the screw axis. Simulation is performed in Isaac Gym. The physics step is 0.005 s and the policy decimation is 10, giving a 20 Hz high-level control rate. Each episode lasts at most 800 steps. The policy action is a normalized $12$-dimensional vector clipped to $[-1,1]$. At each control step, the action is converted to an incremental joint target,
\begin{equation}
    q^{\mathrm{tar}}_{t+1}
    =
    \mathrm{clip}\!(q^{\mathrm{tar}}_t + \alpha\cdot\,a_t,\ q_{\min}, q_{\max}),
\end{equation}
where $\alpha$ is the action scale, and $q_{\min}$ and $q_{\max}$ are the XHand joint limits.

Table~\ref{tab:nutbolt_sim_model} summarizes the domain randomization settings.

\begin{table}[h]
    \centering
    \small
    \caption{Domain Randomization Parameters.}
    \label{tab:nutbolt_sim_model}
    \renewcommand{\arraystretch}{1.12}
    \begin{tabular}{ll}
        \toprule
        Parameter & Range \\
        \midrule
        Object Scale& $\times[0.95, 1.05]$ \\
        Mass & $[0.04, 0.06]$ kg \\
        Center of Mass & $[-0.001, 0.001]$ m \\
        Coefficient of Friction & $[0.5, 8.0]$ \\
        Object Restitution & $[0.0, 1.0]$ \\
        PD Controller Stiffness ($k_p$) & $[2.7, 3.3]$ \\
        PD Controller Damping ($k_d$) & $[0.009, 0.011]$ \\
        Observation Noise (rotation) & $\mathcal{N}(0, 0.01)$ (rad) \\
        Observation Noise (translation) & $\mathcal{N}(0, 0.005)$ (m) \\
        Action Noise (rotation) & $\mathcal{N}(0, 0.01)$ (rad) \\
        Action Noise (translation) & $\mathcal{N}(0, 0.005)$ (m) \\
        External Force Scale & $2.0m$ \\
        External Force Probability & $0.25$ per timestep \\
        \bottomrule
    \end{tabular}
\end{table}

\subsubsection{Oracle Policy}

The oracle policy is trained in simulation with proprioceptive observations, object geometry, and privileged information that is unavailable on the real robot. The proprioceptive input contains the recent joint-position and joint-target history, and the privileged state follows the nut-bolt oracle specification in Table~\ref{tab:nutbolt_priv_info}. 

\begin{table}
    \centering
    \small
    \caption{Privileged information used by the nut-bolt oracle policy.}
    \label{tab:nutbolt_priv_info}
    \renewcommand{\arraystretch}{1.12}
    \begin{tabular}{lc}
        \toprule
        Privileged information & Dimension \\
        \midrule
        Object position & 3 \\
        Object scale & 1 \\
        Object mass & 1 \\
        Object friction coefficient & 1 \\
        Object center of mass & 3 \\
        Object orientation (quaternion) & 4 \\
        Object linear velocity & 3 \\
        Object angular velocity & 3 \\
        Object restitution & 1 \\
        Fingertip positions (2 fingers) & 6 \\
        Fingertip orientations (2 fingers) & 8 \\
        Fingertip linear velocities (2 fingers) & 6 \\
        Fingertip angular velocities (2 fingers) & 6 \\
        Nut contact indicator & 1 \\
        Nut position & 3 \\
        Nut joint velocity & 1 \\
        Nut joint position & 1 \\
        Screw joint friction & 1 \\
        Hand scale & 1 \\
        Hand position & 3 \\
        Hand orientation (quaternion) & 4 \\
        Hand joint positions & 12 \\
        PD controller gains ($k_p$) & 12 \\
        PD controller gains ($k_d$) & 12 \\
        \midrule
        \textbf{Total} & \textbf{97} \\
        \bottomrule
    \end{tabular}
\end{table}

The oracle rewards and weights are listed in
Table~\ref{tab:nutbolt_reward_weights}. The oracle-stage training
schedule is reported separately in Table~\ref{tab:nutbolt_oracle_hyperparams}.

\begin{table}
    \centering
    \small
    \caption{Nut-bolt reward weights.}
    \label{tab:nutbolt_reward_weights}
    \renewcommand{\arraystretch}{1.12}
    \begin{tabular}{lc}
        \toprule
        Reward term & Weight \\
        \midrule
        Rotation reward & $6.0$ \\
        Proximity reward & $2.0$ \\
        Torque penalty & $-0.1$ \\
        Work penalty & $-0.01$ \\
        Pose-difference penalty & $-0.5$ \\
        Large-rotation penalty & $-0.3$ \\
        Point-cloud $z$ penalty & $-1.0$ \\
        \bottomrule
    \end{tabular}
\end{table}

\begin{table}[h]
    \centering
    \small
    \caption{Hyperparameters for training the oracle GenPO++ policy.}
    \label{tab:nutbolt_oracle_hyperparams}
    \renewcommand{\arraystretch}{1.12}
    \begin{tabular}{lc}
        \toprule
        Hyperparameter & Value \\
        \midrule
        \# environments & $8192$ \\
        \# steps & $24$ \\
        \# minibatch size & $32768$ \\
        \# environment steps & $3e9$ \\
        Discount factor ($\gamma$) & $0.99$ \\
        GAE ($\lambda$) & $0.95$ \\
        Learning rate & $1e-4$ \\
        Clip range & $0.2$ \\
        Entropy coefficient & $0.0$ \\
        KL threshold & $0.01$ \\
        Max gradient norm & $1.0$ \\
        \bottomrule
    \end{tabular}
\end{table}

\subsubsection{Sensorimotor Student Policy}

After the oracle policy is trained, we distill it into a deployable sensorimotor policy that does not require privileged simulation state. The student uses only proprioception: a current observation together with a history $h_t\in\mathbb{R}^{30\times 24}$ containing $30$ frames of joint positions and joint targets. Following DexScrew, the student is trained with DAgger: at each step, the sensorimotor policy acts in the environment, while the oracle provides target hand actions and ground-truth privileged embeddings. The training hyperparameters are given in Table~\ref{tab:nutbolt_student_hyperparams}.

\begin{table}[h]
    \centering
    \small
    \caption{Hyperparameters for training the sensorimotor student policy.}
    \label{tab:nutbolt_student_hyperparams}
    \renewcommand{\arraystretch}{1.12}
    \begin{tabular}{lc}
        \toprule
        Hyperparameter & Value \\
        \midrule
        \# environments & $48$ \\
        \# steps & $512$ \\
        \# minibatches & $4096$ \\
        Learning rate & $1e-3$ \\
        \bottomrule
    \end{tabular}
\end{table}

Let $z_t$ denote the privileged embedding used by the oracle and let
$\hat z_t=\phi(h_t)$ be the embedding predicted from proprioceptive
history. The student loss follows the sensorimotor policy training
objective:
\begin{equation}
    \mathcal{L}_{\mathrm{student}}
    =
    \left\|
    a_t^{\mathrm{Hand}}-\hat a_t^{\mathrm{Hand}}
    \right\|_2^2
    +
    \left\|
    z_t-\hat z_t
    \right\|_2^2
    ,
\end{equation}
where $a_t^{\mathrm{Hand}}$ is the oracle target hand action and
$\hat a_t^{\mathrm{Hand}}$ is the action produced by the sensorimotor
policy. The embedding predictor $\phi$ is optimized with Adam until
convergence.

\subsubsection{Deployment}

For real-world deployment, the trained student is exported to TorchScript together with the running mean and variance of the proprioceptive observation and proprioceptive-history normalizers. The XHand controller sends commands at $200$ Hz, while policy inference is performed every $10$ low-level command cycles, matching the $20$ Hz simulation control rate. 

The policy output is clipped to $[-1,1]$, the inactive pinky channels are zeroed for nut-bolt manipulation, and the incremental target rule with the same action scale is applied. To match the $200$ Hz hardware interface and avoid discontinuous jumps in joint targets, the high-level target is not sent as a step command. Instead, we interpolate between the previously commanded target and the new policy target over the next $10$ hardware cycles

\begin{figure}
    \centering
    \includegraphics[width=0.9\linewidth]{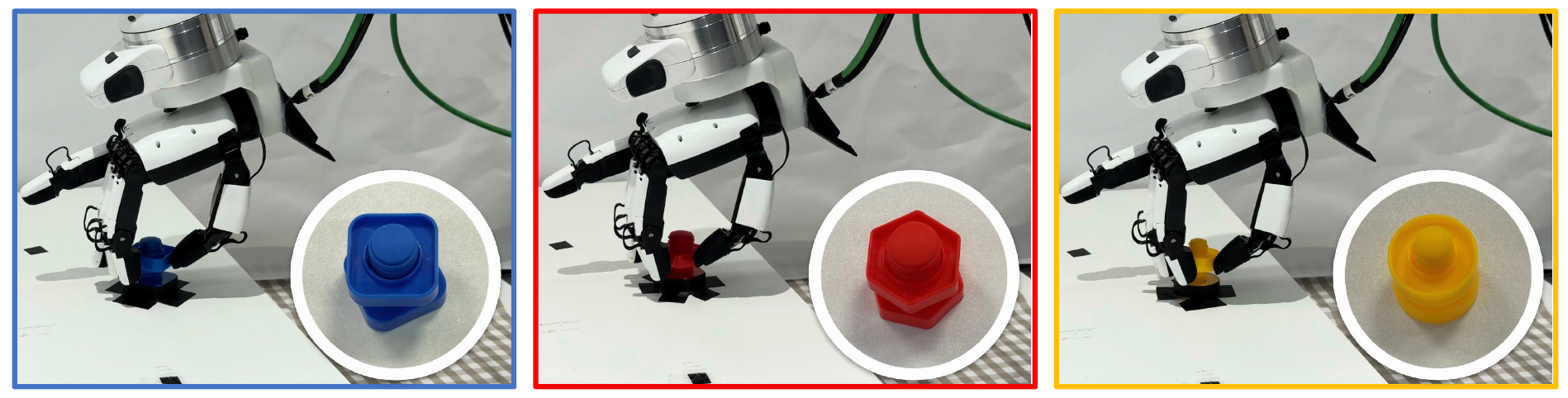}
    \caption{nut-bolts with different geometries.}
    \label{fig:nuts}
\end{figure}

Figure~\ref{fig:nuts} shows nut-bolts with different geometries, which we use in real world deployment. No privileged state, contact signal, nut pose, or simulator parameter is provided, and the robustness across different nut-bolts is provided by the student adaptation module, domain randomization, and mostly the expressiveness of GenPO++.
% \newpage
% \input{checklist.tex}

\end{document}